\documentclass[wcp]{jmlr}

\usepackage{longtable}% for long tables

\usepackage{booktabs}
\makeatletter
\let\Ginclude@graphics\@org@Ginclude@graphics 
\makeatother

\usepackage{bm}

\usepackage{graphicx}

\newtheorem{assumption}{Assumption}

\usepackage{adjustbox}
\usepackage{makecell} % For \thead
\usepackage{siunitx}
\newcommand{\mc}[3]{\multicolumn{#1}{#2}{#3}}

\usepackage{bibentry}
\usepackage{amsfonts}
\usepackage{verbatim}
\usepackage{times}
\usepackage{multirow}
\usepackage{multicol}
\usepackage{kotex}
\usepackage{pifont}
\usepackage{epstopdf}
\usepackage{array}
\usepackage{booktabs}

\usepackage{algorithm}
\usepackage{algorithmicx}

\SetKwInput{KwInput}{Input}                % Set the Input
\SetKwInput{KwInitialize}{Initialization}                % Set the Input
\SetKwInput{KwOutput}{Output}              % set the Output

\title[DMPS]{Diffusion Model Based Posterior Sampling for  Noisy Linear Inverse Problems}

 \author{\Name{Xiangming Meng} \Email{xiangmingmeng@intl.zju.edu.cn}\\
 \addr Zhejiang University - University of Illinois Urbana-Champaign Institute \\
 Zhejiang University \\
 Haining 314400, China
 \AND
 \Name{Yoshiyuki Kabashima} \Email{kaba@phys.s.u-tokyo.ac.jp}\\
\addr Institute for Physics of Intelligence \&  Department of Physics \\
The University of Tokyo\\
Tokyo 113-0033, Japan
}

\begin{document}

\maketitle

% \begin{abstract}
%    We consider the ubiquitous linear inverse problems with additive Gaussian noise and propose an unsupervised sampling approach called diffusion model based posterior sampling (DMPS) to reconstruct the unknown signal from noisy linear measurements. Specifically, using one diffusion model (DM) as an implicit prior, the fundamental difficulty in performing posterior sampling is that the noise-perturbed likelihood score, i.e., gradient of an annealed likelihood function, is intractable. To circumvent this problem, we introduce a simple yet effective closed-form approximation using an uninformative prior assumption.  Extensive experiments are conducted  on a variety of noisy linear inverse problems such as noisy super-resolution, denoising, deblurring, and colorization. In all tasks, the proposed DMPS demonstrates highly competitive or even better performances on various tasks while being 3 times faster than the state-of-the-art competitor diffusion posterior sampling (DPS).
% \end{abstract}

\begin{abstract}
With the rapid development of diffusion models and flow-based generative models, there has been a surge of interests in solving noisy linear inverse problems, e.g., super-resolution, deblurring, denoising, colorization, etc, with generative models. However, while remarkable reconstruction performances have been achieved, their inference time is typically too slow since most of them rely on the seminal diffusion posterior sampling (DPS) framework and thus to approximate the intractable likelihood score, time-consuming gradient calculation through back-propagation is needed. To address this issue, this paper provides a fast and effective solution by proposing a simple closed-form approximation to the likelihood score. For both diffusion and flow-based models, extensive experiments are conducted on various noisy linear inverse problems such as noisy super-resolution, denoising, deblurring, and colorization. In all these tasks, our method (namely DMPS) demonstrates highly competitive or even better reconstruction performances while being significantly faster than all the baseline methods.

% Modern neural networks are extremely large which increases the cost and makes deployment difficult. Common techniques for size reduction rely on quantization which can degrade performance. Here, we directly learn to generate networks where each weight uses just 1-bit. The generation is done via a modified variational auto-encoder whose latent space itself is generated by a diffusion model, initialized with the signs of continuous weights.
% %The latent space reduces the cost associated with previous approaches, while performance is improved using an initialization scheme based on weight-guidance.
% Empirically, compared with traditional techniques for training 1-bit networks, the generated 1-bit weights often show comparable or better performance, but consistently surpasses them on out-of-distribution benchmarks. The work shows that it is possible to save memory but also improve generalization at the same time.

\end{abstract}


\section{Introduction}

Many problems in science and engineering such as computer vision and signal processing can be cast as the following noisy linear inverse problems: 
\begin{align}
    {\bf{y}} = {\bf{Ax}}_0+{\bf{n}}, \label{linear model}
\end{align}
where ${\bf{A}} \in \mathbb{R} ^{M \times N}$ is a (known) linear mixing matrix,  ${\bf{n}}\sim \mathcal{N}({\bf{n}};0,\sigma^2 {\bf{I}})$  is an  i.i.d. additive Gaussian noise, and the goal is to recover the unknown target signal ${\bf{x}}_0\in \mathbb{R} ^{N\times 1}$ from the noisy linear measurements ${\bf{y}}\in \mathbb{R} ^{M\times 1}$. Notable examples include a wide class of  image restoration  tasks like  super-resolution (SR) \cite{ledig2017photo}, colorization \cite{zhang2016colorful}, denoising \cite{buades2005review}, deblurring \cite{yuan2007image}, inpainting \cite{bertalmio2000image}, as well as the well-known compressed sensing  (CS) \cite{candes2006robust,candes2008introduction} in signal processing. One big challenge of these linear inverse problems is that they are  ill-posed \cite{o1986statistical}, i.e., the solution to (\ref{linear model}) is not unique (even in the noiseless case). This problem can be tackled from a Bayesian perspective: suppose that the target signal $\bf{x}$ follows a proper \textit{prior distribution} $p(\bf{x})$, given noisy observations $\bf{y}$, one can perform posterior  sampling from  $p({\bf{x}}_0|\bf{y})$ to recover ${\bf{x}}_0$. Hence, an accurate prior $p({\bf{x}}_0)$ is crucial in recovering ${\bf{x}}_0$. Various kinds of priors or structure constraints have been proposed, including sparsity \cite{candes2008introduction},  low-rank \cite{fazel2008compressed}, total variation \cite{candes2006robust}, just to name a few. However, such handcrafted priors might fail to capture the capture more rich structure of natural signals \cite{ulyanov2018deep}.

\begin{figure}[t!]
\centering  
\includegraphics[width=0.9\textwidth]{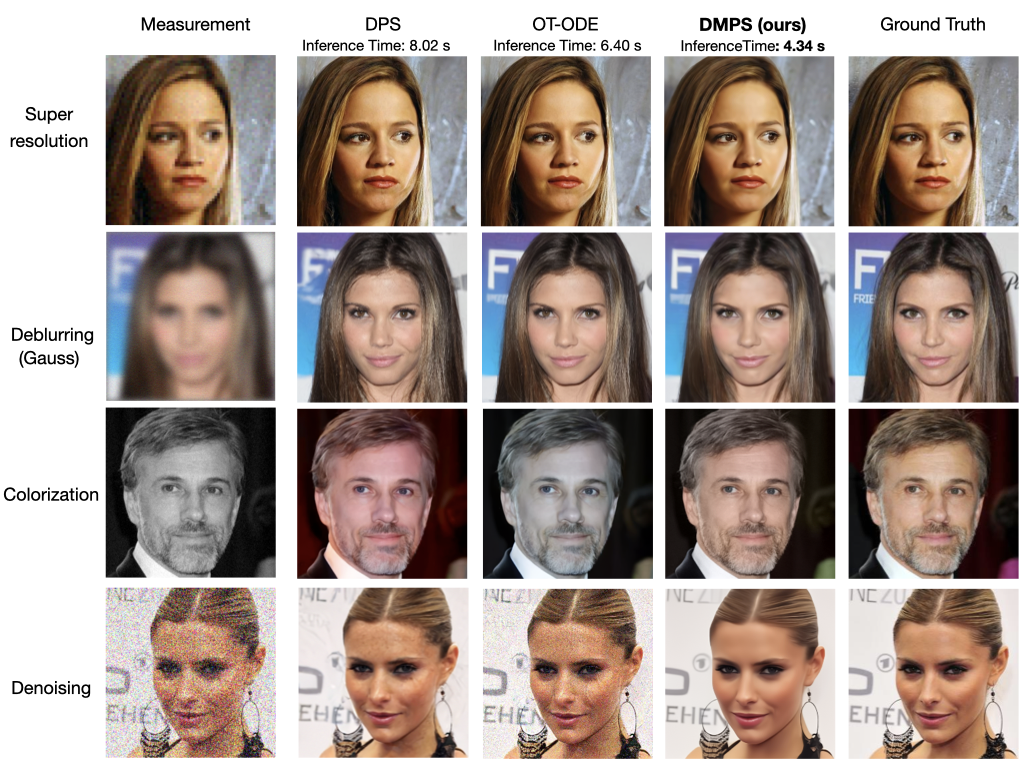}
\caption{Typical results for different image restoration tasks on CelebA-HQ $256\times 256$ validation set, along with the average inference time in seconds. It can be seen that our method (DMPS) achieves highly competitive or even better reconstruction performances with much less inference time. For a fair comparison, all the algorithms are run on the same flow-based model with NFE=50. }
\label{fig:flow-results}
\end{figure}

With the recent advent of  diffusion models \cite{sohl2015deep, song2019generative,ho2020denoising,dhariwal2021diffusion,rombach2022high} and flow-based models \cite{lipman2022flow,liu2022flow,albergo2023stochastic,ma2024sit}, there has been a surge of interests in applying them to solve the linear inverse problems with  remarkable performances \cite{kadkhodaie2020solving, kadkhodaie2021stochastic, jalal2021robust, jalal2021instance,kawar2021snips,kawar2022denoising,  chung2022improving, chung2022diffusion, wang2022zero, meng2023quantized, meng2024qcs,pokle2023training}. One fundamental challenge in this field is computing the score of \textit{noise-perturbed likelihood}  $p({\bf{y}}|{\bf{x}}_t)$, i.e., $\nabla_{{\bf{x}}_t} \log{p({\bf{y}}|{\bf{x}}_t)}$, where ${\bf{x}}_t$ is a noise-perturbed version of ${\bf{x}}_0$ at time instance $t$ defined by the forward process of DM \cite{ho2020denoising, song2019generative}. This is because while $\nabla_{{\bf{x}}_t} \log{p({\bf{y}}|{\bf{x}}_t)}$ is easily obtained for $t=0$ from (\ref{linear model}), it is intractable for general $t>0$. To address this challenge, most diffusion and flow-based methods adopt the diffusion posterior sampling (DPS) framework \cite{chung2022diffusion} which leverages the Tweedie’s formula \cite{robbins1992empirical} to obtain a posterior estimate of $\bf{x}_0$. While DPS and its variants achieve excellent reconstruction performances, they suffer from  a big disadvantage that their inference speed is very slow due to the time-consuming gradient calculation through back-propagation. 
 % Hence, the distinguishing factor among various algorithms \cite{jalal2021robust, jalal2021instance, kawar2021snips, kawar2022denoising, chung2022improving, chung2022diffusion, wang2022zero, meng2023quantized} lies in how they address the intractability of $\nabla_{{\bf{x}}_t} \log{p({\bf{y}}|{\bf{x}}_t)}$. Notably, two SOTA methods are the renowned diffusion posterior sampling (DPS) \cite{chung2022diffusion} employing a Laplace approximation, and denoising diffusion null-space Model (DDNM) \cite{wang2022zero} with an iterative null-space refinement. Despite their superiority, both DPS and DDNM exhibit certain limitations, i.e., DPS is hindered by very slow inference due to an involved gradient operation and DDNM is quite susceptible to additive noise (\ref{linear model}). Although a robust version of DDNM, termed as  DDNM+,   was  proposed \cite{wang2022zero} to tackle noise, it can still lead to bad results for tasks like deblurring. 

In this paper, we take an alternative perspective and provide a simple fast solution for solving the noisy linear inverse problems with diffusion or flow-based models by proposing a closed-from approximation to the intractable function \(\nabla_{{\bf{x}}_t} \log{p({\bf{y}}|{\bf{x}}_t)}\). Our primary goal is to reduce the inference time of existing methods with minimal degradation, rather than to compete with state-of-the-art performance. The key observation is that, the noise-perturbed likelihood \(p({\bf{y}}|  {{\bf{x}}}_t) = \int   p({\bf{y}} | {\bf{x}}_0)   p({\bf{x}}_0|  {\bf{x}}_t ) d{\bf{x}}_0\) is unavailable due to the intractability of the  reverse transition probability $p({\bf{x}}_0|  {\bf{x}}_t )$, so that one can  obtain a closed-form approximation of it assuming an \textit{uninformative} prior  $p({\bf{x}}_0)$.  Interestingly, such assumption is asymptotically accurate when the perturbed noise in ${\bf{x}}_t$ negligibly small. The resultant algorithm is denoted as  Diffusion Model based Posterior Sampling (DMPS), one approach that applies to both diffusion and flow-based models. Compared with the  seminal DPS and its variants such as PGDM, thanks to the proposed closed-from approximation, no back-propagation through the pre-trained model is needed, thus significantly reducing the inference time. To verify its efficacy,  a variety of experiments on different linear inverse problems such as  image super-resolution, denoising, deblurring, colorization, are conducted.  Remarkably, as shown in  Figure \ref{fig:flow-results}, in all these tasks, despite its simplicity, DMPS achieves highly competitive or even better reconstruction performances, while the running time is significantly reduced.

\section{Background}
Diffusion models (DM) \cite{song2019generative,ho2020denoising,dhariwal2021diffusion,song2023consistency} and  Flow-based models (such as flow matching, rectified flow) \cite{lipman2022flow,liu2022flow,albergo2023stochastic,ma2024sit} can be seen as a unified class of probabilistic generative models that learn to turning random noise into data samples  ${\bf{x}}_0\sim p({\bf{x}}_0)$. The forward time-dependent process ${\bf{x}}_0\to {\bf{x}}_1\to \cdots \to {\bf{x}}_T$ can be described as follows:
\begin{align}
{\bf{x}}_t = a_t {\bf{x}}_0 + b_t {\bf{\epsilon}},
\label{eq-forward}
\end{align}
where $a_t$ is a decreasing function of $t$, $b_t$ is an increasing function of $t$, and  ${\bf{\epsilon}} \sim \mathcal{N}(\bf{0}, \bf{I})$ is an i.i.d. standard Gaussian noise. Equivalently, the forward process (\ref{eq-forward}) is modeled as
\begin{align}
    p({\bf{x}}_t|{\bf{x}}_0)= \mathcal{N}({\bf{x}}_t;a_t{\bf{x}}_0, b_t^2{\bf{I}}). \label{forward-process}
\end{align}
Both diffusion models and flow-based models aim to reverse the forward process (\ref{eq-forward}) and generate new samples from a distribution that approximates the target data distribution $p({\bf{x}}_0)$. 

\noindent {\bf{Diffusion Models}}:  Diffusion models reverse the forward process (\ref{eq-forward}) by performing a denoising task for each step, i.e., predicting the noise $\epsilon$ from ${\bf{x}}_t$. In the  seminal work of  DDPM \cite{ho2020denoising}, $a_t=\sqrt{{\bar{\alpha}}_t},\; b^2_t=1-\bar{\alpha}_t$, where ${\bar{\alpha}}_t=\prod_{i=1}^t \alpha_i$,  $\alpha_t = 1-\beta_t$, and $0<\beta_1<\beta_1< \cdots < \beta_T<1$ \cite{ho2020denoising}. Denote ${\rm{s}}_{\boldsymbol{\theta}}({\bf{x}}_{t},t)$ as the noise approximator from ${\bf{x}}_t$, one can generate samples following the estimated reverse process  \cite{ho2020denoising} as
\begin{align}
    {\bf{x}}_{t-1} = \frac{1}{\sqrt{\alpha_t}} \big( {\bf{x}}_{t} - \frac{1-\alpha_t}{\sqrt{1-\bar{\alpha}}_t} {\rm{s}}_{\boldsymbol{\theta}}({\bf{x}}_{t},t)\big) + {\beta}_t {\bf{z}}_t, \label{reverse-process}
\end{align}
where ${\bf{z}}_t \sim \mathcal{N}(\bf{0}, \bf{I})$ is an i.i.d. standard Gaussian noise. Note that in the variant ADM in \cite{dhariwal2021diffusion}, the reverse noise variance $\beta_t$ is learned as  $\{\tilde{\sigma}_t\}_{t=1}^T$, which further improves the performances of DDPM. 

Diffusion models are also known as score-based generated models since the denoising process is equivalent to approximating the score function  $\nabla_{{\bf{x}}_t} \log{{{p}}({\bf{x}}_t})$ \cite{song2019generative,song2020denoising}. For example, for DDPM, there is a one-to-one mapping between  ${\rm{s}}_{\boldsymbol{\theta}}({\bf{x}}_{t},t)$  and  $\nabla_{{\bf{x}}_t} \log{{{p}}({\bf{x}}_t})$  
\begin{align}
  \nabla_{{\bf{x}}_t} \log{{{p}}({\bf{x}}_t}) = -\frac{1}{\sqrt{1-\bar{\alpha}}_t} {\rm{s}}_{\boldsymbol{\theta}}({\bf{x}}_{t},t). \label{score-relation}
\end{align}

\noindent {\bf{Flow-based Models}}: Flow-based models can be viewed as a generalization of diffusion models \cite{lipman2022flow,liu2022flow,albergo2023stochastic,ma2024sit}, which introduce a probability ODE with a velocity field \cite{lipman2022flow,ma2024sit}：
\begin{equation}
\dot{\mathbf{x}}_t = \mathbf{v}(\mathbf{x}_t, t),
\label{eq-ODE-flow}
\end{equation}
where $\mathbf{v}(\mathbf{x}, t)$ can be obtained as the conditional expectation $\mathbf{v}(\mathbf{x}, t)=\mathbb{E}[\dot{\mathbf{x}}_t | \mathbf{x}_t = \mathbf{x}]$. Flow-based models solve the probability ODE (\ref{eq-ODE-flow}) backwards by learning the velocity field $\mathbf{v}(\mathbf{x}, t)$ using a neural network ${\rm{v}}_{\boldsymbol{\theta}}({\bf{x}},t)$, and a first-order ODE solver can be realized as follows:
\begin{equation}
\mathbf{x}_{t-1} = \mathbf{x}_{t} -{\rm{v}}_{\boldsymbol{\theta}}({\bf{x}}_{t},t)\Delta_t,
\label{eq-ODE-sampler}
\end{equation}
where $\Delta_t$ is the sampling time interval. 
Interestingly, the score function $ \nabla_{{\bf{x}}_t} \log{{{p}}({\bf{x}}_t})$ can also be expressed in terms of the velocity field \cite{ma2024sit}
\begin{equation}
\nabla_{{\bf{x}}_t} \log{{{p}}({\bf{x}}_t}) = b_t^{-1} \frac{a_t {\rm{v}}_{\boldsymbol{\theta}}({\bf{x}}_{t},t) - \dot{a}_t {\mathbf{x}}_t}{\dot{a}_t b_t - a_t \dot{b}_t}.
\label{score-relation-flow}
\end{equation}

\noindent {\bf{Previous Methods with Diffusion and Flow-based Models}}: 
The problem of reconstructing ${\bf{x}}_0$ from noisy $\bf{y}$ in (\ref{linear model}) can be cast as performing \textit{posterior} inference, i.e., 
\begin{align}
   p({\bf{x}}_0| {\bf{y}}) = \frac{p({\bf{x}}_0)p(\bf{y}|{\bf{x}}_0)}{p({\bf{y}})}, \label{posterior-dist}
\end{align}
where $p({\bf{x}}_0 | {\bf{y}})$ is the \textit{posterior} distribution. 
Ideally, one can directly train diffusion or flow-based models using samples from $p({\bf{x}}| {\bf{y}})$. However, such a supervised approach is neither efficient nor flexible and most previous methods adopt an unsupervised approach \cite{jalal2021robust,chung2022diffusion,song2022pseudoinverse,pokle2023training}: given a pre-trained  diffusion model or flow-based model, one treats it as an implicit prior $p({\bf{x}}_0)$ and then performs posterior sampling through a reverse sampling process ${\bf{x}}_T\to \cdots {\bf{x}}_{t}\to {\bf{x}}_{t-1}\to \cdots \to {\bf{x}}_0$. The main challenge is thus how to incorporate information of $\bf{y}$ within such reverse sampling process. Interestingly, while diffusion models and flow-based models admit slightly different forms, there exists a principled way thanks to the simple relation from the Bayes' rule (\ref{posterior-dist}),
\begin{align}
    \nabla_{{\bf{x}}_t} \log{p({\bf{x}}_t| {\bf{y}})} =  \nabla_{{\bf{x}}_t} \log{p({\bf{x}}_t)} +  \nabla_{{\bf{x}}_t} \log{p({\bf{y}}| {\bf{x}}_t}),   \label{bayes-rule-score}
\end{align}
where $p({{\bf{x}}_t}| {\bf{y}})$  is the score of posterior distribution (we call \textit{posterior score}), which is the sum of the  prior score $\nabla_{{\bf{x}}_t} \log{p({\bf{x}}_t)} $, and the likelihood score $\nabla_{{\bf{x}}_t} \log{p({\bf{y}}| {\bf{x}}_t})$. Given a pre-trained  diffusion model or flow-based model, the prior score $\nabla_{{\bf{x}}_t} \log{p({\bf{x}}_t)}$ can be  readily obtained from the pre-trained model outputs thanks to the intrinsic connections (\ref{score-relation}) (\ref{score-relation-flow}). However, while $\nabla_{{\bf{x}}_t} \log{p({\bf{y}}| {\bf{x}}_t})$ can be readily obtained from (\ref{linear model}) when $t=0$, it becomes intractable in the general case for $t>0$ \cite{chung2022diffusion}. To see this,  one can equivalently write ${p({\bf{y}}| {\bf{x}}_t})$ as 
\begin{align}
   p({\bf{y}} | {{\bf{x}}}_t) = \int   p({\bf{y}} | {\bf{x}}_0)   p({\bf{x}}_0 | {\bf{x}}_t ) d{\bf{x}}_0,\label{eq:rigorous-likelihood-def}
\end{align}
where from the Bayes' rule,
\begin{align}
   p({\bf{x}}_0 | {\bf{x}}_t )= \frac{p({\bf{x}}_t  | {\bf{x}}_0) p({\bf{x}}_0) }{\int p({\bf{x}}_t  | {\bf{x}}_0) p({\bf{x}}_0)  d{\bf{x}}_0 }. \label{eq:pdf_condition_x}
\end{align}
For both diffusion and flow-based models, although the forward transition probability $p({\bf{x}}_t  | {\bf{x}}_0)$ is exactly known as (\ref{forward-process}), the reverse transition probability $ p({\bf{x}}_0 | {\bf{x}}_t )$ is difficult to obtain. Consequently, the remaining key challenge is the calculation of the \textit{noise-perturbed likelihood score} $\nabla_{{\bf{x}}_t} \log{p({\bf{y}}| {\bf{x}}_t})$. A variety of methods \cite{jalal2021robust,chung2022diffusion,song2022pseudoinverse,pokle2023training} have been proposed to approximate $\nabla_{{\bf{x}}_t} \log{p({\bf{y}}| {\bf{x}}_t})$ (or equivalently ${p({\bf{y}}| {\bf{x}}_t})$) and most of them build on the seminal work DPS \cite{chung2022diffusion} which leverages the Tweedie’s formula to obtain the posterior estimate of ${\bf{x}}_0$ \cite{robbins1992empirical, chung2022diffusion}：
\begin{align}
\hat{\bf{x}}_0({\bf{x}}_t) := \mathbb{E}[{\bf{x}}_0 | {\bf{x}}_t] = \frac{1}{a_t} \left( {\bf{x}}_t + b_t^2 \nabla_{{\bf{x}}_t} \log p_t({\bf{x}}_t) \right),
\label{eq:tweedie-formula}
\end{align}
where $\nabla_{{\bf{x}}_t} \log p_t({\bf{x}}_t)$ is approximated by the neural network as (\ref{score-relation}) and (\ref{score-relation-flow}) for diffusion and flow-based models, respectively. 
In particular, DPS uses a Laplace approximation $p({\bf{y}}| {\bf{x}}_t)\approx p({\bf{y}}| \hat{\bf{x}}_0({\bf{x}}_t))=\mathcal{N}({\bf{A}}\hat{\bf{x}}_0({\bf{x}}_t);\sigma_y^2{\bf{I}})$, while both PGDM \cite{song2022pseudoinverse} and OT-ODE \cite{pokle2023training} use an approximation $p({\bf{y}}| {\bf{x}}_t) \approx \mathcal{N}({\bf{A}}\hat{\bf{x}}_0({\bf{x}}_t);\gamma^2_t {\bf{A}}{\bf{A}}^T+\sigma_y^2{\bf{I}})$, where $\gamma_t$ is a hyper-parameter for the variance term. Nevertheless, while DPS and its variants can achieve excellent reconstruction performances, they suffer from a significant drawback: the inference speed is very slow due to the time-consuming gradient of the pre-trained model output w.r.t.  ${\bf{x}}_t$ in calculating the likelihood $\nabla_{{\bf{x}}_t} \log{p({\bf{y}}| {\bf{x}}_t})$.

% Diffusion models (DM)  are a novel class of probabilistic generative models originated from non-equilibrium thermodynamics \cite{sohl2015deep}. The basic idea of DM is to gradually perturb the data ${\bf{x}}_0 = \bf{x}$ with different noise scales in the forward process and then learn to reverse the diffusion process to sample ${\bf{x}}_0\sim p({\bf{x}}_0)$ from the noise \cite{sohl2015deep}. Various kinds of DM \cite{song2019generative,ho2020denoising,dhariwal2021diffusion,song2023consistency} have been proposed  and in the following we specialize our focus on the seminal denoising diffusion probabilistic model (DDPM) \cite{ho2020denoising}, in particular the ablated diffusion model (ADM) in \cite{dhariwal2021diffusion}. 

% Specifically, given data samples ${\bf{x}}_0\sim p({\bf{x}}_0)$, the forward diffusion process ${\bf{x}}_0\to {\bf{x}}_1\to \cdots \to {\bf{x}}_T$ is defined as \cite{ho2020denoising}
% \begin{align}
%     p({\bf{x}}_t|{\bf{x}}_{t-1}) =  \mathcal{N}({\bf{x}}_t;\sqrt{1 - {{\beta}}_t} {\bf{x}}_0, \beta_t{\bf{I}}), t=1\cdots T, \label{forward-process}
% \end{align}
% where $0<\beta_1<\beta_1< \cdots < \beta_T<1$ is prescribed perturbed noise variances so that approximately ${\bf{x}}_T\sim \mathcal{N}({\bf{0}}, {\bf{I}})$. Denote as $\alpha_t = 1-\beta_t$ and ${\bar{\alpha}}_t=\prod_{i=1}^t \alpha_i$, one can easily obtain that  \cite{ho2020denoising}
% \begin{align}
%     p({\bf{x}}_t|{\bf{x}}_0)= \mathcal{N}({\bf{x}}_t;\sqrt{{\bar{\alpha}}_t} {\bf{x}}_0, (1-\bar{\alpha}_t){\bf{I}}). \label{forward-process}
% \end{align}

\section{Method}
\label{Sec-method}
In this section, we propose a fast and efficient closed-form solution for the intractable likelihood score $\nabla_{{\bf{x}}_t} \log{p({\bf{y}}| {\bf{x}}_t})$, termed as  noise-perturbed pseudo-likelihood score. We first derive the results of $\nabla_{{\bf{x}}_t} \log{p({\bf{y}}| {\bf{x}}_t})$ under the general settings (\ref{eq-forward}-\ref{forward-process}), and then apply our results in diffusion and flow-based models, respectively. 

\subsection{Noise-Perturbed Pseudo-Likelihood Score}
To tackle the intractability difficulty of $\nabla_{{\bf{x}}_t} \log{p({\bf{y}}| {\bf{x}}_t})$, we introduce a simple approximation under the following assumption:  
\begin{assumption} (uninformative prior)
\label{uninformative-assumption}
The prior $p({\bf{x}}_0)$ (\ref{eq:pdf_condition_x}) is \textit{uninformative} (flat) w.r.t. ${\bf{x}}_t$ so that $p({\bf{x}}_0 | {\bf{x}}_t) \propto  p({\bf{x}}_t  | {\bf{x}}_0)$, where $\propto $ denotes equality up to a constant scaling. 
\end{assumption} 
Note that while the uninformative prior assumption appears crude at first sight, it is asymptotically accurate when the perturbed noise in ${\bf{x}}_t$ becomes negligible, as verified in Appendix \ref{appendix:uninformative-prior}.

Under Assumption \ref{uninformative-assumption}, we obtain a simple \textit{closed-form}  approximation of $\nabla_{{\bf{x}}_t} \log{p({\bf{y}}| {\bf{x}}_t})$ called noise-perturbed \textit{pseudo-likelihood score} and denote as $\nabla_{{\bf{x}}_t} \log{{\tilde{p}}({\bf{y}}| {\bf{x}}_t})$, as shown in Theorem \ref{noise-likelohood-score-linear-VPSDE}. 

\begin{theorem} (noise-perturbed pseudo-likelihood score for (\ref{eq-forward}))
\label{noise-likelohood-score-linear-VPSDE}
For the general forward process (\ref{eq-forward}),  under Assumption \ref{uninformative-assumption}, the noise-perturbed likelihood score $\nabla_{{\bf{x}}_t} \log{p({\bf{y}}| {\bf{x}}_t})$ for  $\bf{y} = {\bf{Ax}_0+n}$ in (\ref{linear model}) admits a closed-form 
\begin{align}
    &\nabla_{{\bf{x}}_t} \log{p({\bf{y}}| {\bf{x}}_t}) \simeq \nabla_{{\bf{x}}_t} \log{{\tilde{p}}({\bf{y}}| {\bf{x}}_t}) \nonumber \\
    =&\frac{1}{a_t} {\bf{A}}^T {\Big(\sigma_y^2{\bf{I}}+ \frac{b^2_t}{a^2_t} {\bf{A}}{\bf{A}}^T\Big)^{-1}} \big({\bf{y}} - \frac{1}{a_t} {{\bf{A{\bf{x}}}}_t}\big). \label{eq:likelihood-score-original}
\end{align}
\end{theorem} 
$\bf{Proof}$. 
From Assumption \ref{uninformative-assumption}, we have $p({\bf{x}}_0 | {\bf{x}}_t) \propto p({\bf{x}}_t | {\bf{x}}_0)$. Recall that for the forward process (\ref{eq-forward}), $p({\bf{x}}_t | {\bf{x}}_0)$ is Gaussian (\ref{eq-forward}). By completing the squares w.r.t. ${\bf{x}}_0$, an approximation for $p({\bf{x}}_0 | {\bf{x}}_t)$ can be derived as follows:
\begin{align}
    p({\bf{x}}_0 | {\bf{x}}_t) \simeq  \mathcal{N}({\bf{x}}_0; \frac{{\bf{x}}_t}{a_t}, \frac{b_t^2}{a_t^2}\bf{I}),
    \label{eq:reverse_prob_approx}
\end{align}
whereby ${\bf{x}}_0$ can be  equivalently written as ${\bf{x}}_0  =  \frac{{\bf{x}}_t}{a_t} + {\frac{b_t}{{a_t}} } {\bf{w}}$, where $\bf{w}\sim \mathcal{N}(\bf{0,I})$. Thus, from (\ref{linear model}), we obtain an alternative representation of $ {\bf{y}}$
\begin{align}
    {\bf{y}} = \frac{{\bf{Ax}}_t}{a_t} + {\frac{b_t}{{a_t}} } {\bf{Aw}} + \bf{n}. \label{linear_model_equiv}
\end{align}
After some simple algebra, the likelihood ${p}({\bf{y}}|{\bf{x}}_t)$ can be approximated as $\tilde{p}({\bf{y}}|{\bf{x}}_t)$ 
\begin{align}
    \tilde{p}({\bf{y}}|{\bf{x}}_t) =  \mathcal{N}({\bf{y}}; \frac{{\bf{Ax}}_t}{a_t},\sigma_y^2{\bf{I}} + \frac{b^2_t}{a_t^2} {{\bf{AA}}^T}), \label{PL-result}
\end{align}
where $\tilde{p}({\bf{y}}|{\bf{x}}_t)$ is used to denote the \textit{pseudo-likelihood} as opposed to the exact ${p}({\bf{y}}|{\bf{x}}_t)$ due to Assumption \ref{uninformative-assumption}. 
Using (\ref{PL-result}), one can readily obtain a closed-form solution for the noise-perturbed pseudo-likelihood score $\nabla_{{{\bf{x}}_t}} \log \tilde{p}({\bf{y}}|{\bf{x}}_t)$ as (\ref{eq:likelihood-score-original}), 
which completes the proof. $\hfill\blacksquare$

As shown in Theorem \ref{noise-likelohood-score-linear-VPSDE}, now we obtain a simple \textit{closed-form} approximation for the intractable likelihood score, which is much easier to implement compared to DPS and its variants. In the special case when  ${\bf{A}}$ itself is row-orthogonal, i.e.,  ${\bf{A}}{\bf{A}}^T$ is diagonal, the matrix inversion is trivial and  (\ref{eq:likelihood-score-original}) simply reduces to 
\begin{align}
   [\nabla_{{\bf{x}}_t} \log{\tilde{p}({\bf{y}}| {\bf{x}}_t})]_m =\frac{{\bf{a}}^T_m\left({\bf{y}} - \frac{1}{a_t}  {\bf{A}}{\bf{x}}_t\right)}{\sigma_y^2 {a_t}  + \frac{b_t^2}{a_t^2} \left\Vert {\bf{a}}_{m}\right\Vert _{2}^{2}}. \label{likelihood-score-linear}
\end{align}
where $[\cdot]_m$ is the $m$-th element and  ${\bf{a}}_{m}$ is the $m$-th row of $\bf{A}$.  For general matrices $\bf{A}$, such an inversion is essential but it can also be efficiently implemented by resorting to singular value decomposition (SVD) of $\bf{A}$, as shown in Theorem \ref{noise-likelohood-score-linear-VPSDE-SVD}.
\begin{corollary} (efficient computation via SVD)
\label{noise-likelohood-score-linear-VPSDE-SVD}
For the general forward process (\ref{eq-forward}), the noise-perturbed pseudo-likelihood score $\nabla_{{\bf{x}}_t} \log{p({\bf{y}}| {\bf{x}}_t})$ in (\ref{eq:likelihood-score-original}) of Theorem \ref{noise-likelohood-score-linear-VPSDE} can be equivalently computed as
\begin{align}
    &\nabla_{{\bf{x}}_t} \log{p({\bf{y}}| {\bf{x}}_t}) \simeq \nabla_{{\bf{x}}_t} \log{{\tilde{p}}({\bf{y}}| {\bf{x}}_t})  \nonumber \\
    =&\frac{1}{a_t} {\bf{V\Sigma}} {\Big(\sigma_y^2{\bf{I}}+ \frac{b_t^2}{a^2_t} {\bf{\Sigma}}^2 \Big)^{-1}} {\bf{U}}^T\big({\bf{y}} - \frac{1}{{a_t}} {\bf{A}} {{{\bf{x}}}_t}\big),\label{eq:likelihood-score-SVD}
\end{align}
where ${\bf{A} = U\Sigma V}^T$ is the SVD of $\bf{A}$ and ${\bf{\Sigma}}^2$ denotes element-wise square of $\bf{\Sigma}$.
\end{corollary}
% \begin{proof}
% The result is straightforward from Theorem \ref{noise-likelohood-score-linear-VPSDE}. 
% \end{proof}
$\bf{Proof}$. The result is straightforward from Theorem \ref{noise-likelohood-score-linear-VPSDE}. $\hfill\blacksquare$
\begin{remark}
Thanks to SVD, there is no need to compute the matrix inversion in (\ref{eq:likelihood-score-original}) for each $t$. Instead, one simply needs to perform SVD of $\bf{A}$ only once and then compute $\nabla_{{\bf{x}}_t} \log{{\tilde{p}}({\bf{y}}| {\bf{x}}_t})$ by  (\ref{eq:likelihood-score-SVD}), which is quite simple since $\bf{\Sigma}$ is a diagonal matrix.  
\end{remark}

% \begin{algorithm}[t]
%     \caption{DMPS (DDPM version)}
    
%     \label{alg: posterior-sampling-algorithm}
%     \DontPrintSemicolon
%       \KwInput{$\bf{y,A}$, $\sigma^2$, $\{\tilde{\sigma}_t\}_{t=1}^T,\lambda$}
%       \KwInitialize{${\bf{x}}_T\sim \mathcal{N}(\bf{0}, \bf{I})$, ${\bf{A} = U\Sigma V}^T$}
%       \For{$t=T$ {\bfseries to} $1$}{
%             Draw ${\bf{z}}_t \sim \mathcal{N}(\bf{0}, \bf{I})$ 
            
%             ${\bf{x}}_{t-1} = \frac{1}{\sqrt{\alpha_t}} \big( {\bf{x}}_{t} - \frac{1-\alpha_t}{\sqrt{1-\bar{\alpha}}_t} {\rm{s}}_{\boldsymbol{\theta}}({\bf{x}}_{t},t)\big) + \tilde{\sigma}_t {\bf{z}}_t$  
      
%             {{$\nabla_{{\bf{x}}_t} \log{\tilde{p}({\bf{y}}| {\bf{x}}_t}) = \frac{1}{\sqrt{{\bar{\alpha}}_t}} {\bf{V\Sigma}} {\Big(\sigma^2{\bf{I}}+ \frac{1-{\bar{\alpha}}_t}{{\bar{\alpha}}_t} {\bf{\Sigma}}^2 \Big)^{-1}} \big({\bf{U}}^T{\bf{y}} - \frac{1}{\sqrt{{\bar{\alpha}}_t}} {\bf{\Sigma V}}^T {{{\bf{x}}}_t}\big)$} }     
            
%             ${{\bf{x}}_{t-1} = {\bf{x}}_{t-1} + \lambda \frac{1-\alpha_t}{\sqrt{\alpha_t}} \nabla_{{\bf{x}}_t} \log{\tilde{p}({\bf{y}}| {\bf{x}}_t})}$
%        }
%     \KwOutput{${\bf{{x}}}_0$}
% \end{algorithm}

\subsection{DMPS: Diffusion Model Based Posterior Sampling}
\label{sec:DMPS}
After obtaining the approximate results of the likelihood score function $\nabla_{{\bf{x}}_t} \log{p({\bf{y}}| {\bf{x}}_t})$, we can easily modify the sampling equations of the original diffusion and flow-based models from a unified Bayesian perspective. Here we introduce a simple yet universal procedure demonstrating how we can achieve this for both diffusion and flow-based models. 

\noindent \textbf{Step 1}: Reformulate the original sampling equations for unconditional generation in the terms of the prior score $\nabla_{{\bf{x}}_t} \log{{{p}}({\bf{x}}_t})$.
This step requires building connections between the generative model (either diffusion or flow-based models) output with the score function $\nabla_{{\bf{x}}_t} \log{{{p}}({\bf{x}}_t})$. For example, given the connections (\ref{score-relation}) (\ref{score-relation-flow}), the original sampling equation  (\ref{reverse-process}) for DDPM and (\ref{eq-ODE-sampler})  for flow-based models can be rewritten using $\nabla_{{\bf{x}}_t} \log{{{p}}({\bf{x}}_t})$ as follows 
\begin{align}
    \text{DDPM:\;\;} {\bf{x}}_{t-1} =& \frac{1}{\sqrt{\alpha_t}} \big( {\bf{x}}_{t} + ({1-\alpha_t})\nabla_{{\bf{x}}_t} \log{{{p}}({\bf{x}}_t})\big) + {\beta}_t {\bf{z}}_t, \label{reverse-process-score-form-ddpm}\\
    \text{Flow-based:\;\;} \mathbf{x}_{t-1} =& \mathbf{x}_{t} - \big( \frac{\dot{a}_t}{a_t}{\bf{x}}_{t} + \frac{b_t(\dot{a}_t b_t - a_t \dot{b}_t)}{a_t}\nabla_{{\bf{x}}_t} \log{{{p}}({\bf{x}}_t})   \big)\Delta_t, 
    \label{reverse-process-score-form}
\end{align}

\begin{align}
    \text{DDPM:\;\;} {\bf{x}}_{t-1} =& \frac{1}{\sqrt{\alpha_t}} \big( {\bf{x}}_{t} + ({1-\alpha_t}) {\textcolor{red}{\nabla_{{\bf{x}}_t} \log{{{p}}({\bf{x}}_t}) }{\textcolor{blue}{+\nabla_{{\bf{x}}_t} \log{p({\bf{y}}| {\bf{x}}_t})}}) \big)} + {\beta}_t {\bf{z}}_t, \label{reverse-process-score-form-ddpm}\\
    \text{Flow-based:\;\;} \mathbf{x}_{t-1} =& \mathbf{x}_{t} - \big( \frac{\dot{a}_t}{a_t}{\bf{x}}_{t} + \frac{b_t(\dot{a}_t b_t - a_t \dot{b}_t)}{a_t} {\textcolor{red}{\nabla_{{\bf{x}}_t} \log{{{p}}({\bf{x}}_t})} } {\textcolor{blue}{+\nabla_{{\bf{x}}_t} \log{p({\bf{y}}| {\bf{x}}_t})}}) \big)\Delta_t, 
\end{align}

\noindent \textbf{Step 2}:  Replace the prior score $\nabla_{{\bf{x}}_t} \log{{{p}}({\bf{x}}_t})$ in the sampling equations obtained in Step 1 with the posterior score  $\nabla_{{\bf{x}}_t} \log{p({\bf{x}}_t| {\bf{y}})}$ as (\ref{bayes-rule-score}). For example, for DDPM and flow-based models, the corresponding sampling equations (\ref{reverse-process-score-form-ddpm}-\ref{reverse-process-score-form}) become
\begin{align}
    \text{DDPM:\;\;}{\bf{x}}_{t-1} =& \frac{1}{\sqrt{\alpha_t}} \big( {\bf{x}}_{t} + ({1-\alpha_t}) (\nabla_{{\bf{x}}_t}\log{{{p}}({\bf{x}}_t}){\textcolor{blue}{+\nabla_{{\bf{x}}_t} \log{p({\bf{y}}| {\bf{x}}_t})}})\big) + {\beta}_t {\bf{z}}_t,\label{reverse-process-post-form-ddpm}\\
    \text{Flow-based:\;\;}  \mathbf{x}_{t-1} =& \mathbf{x}_{t} - \big( \frac{\dot{a}_t}{a_t}{\bf{x}}_{t} + \frac{b_t(\dot{a}_t b_t - a_t \dot{b}_t)}{a_t}(\nabla_{{\bf{x}}_t} \log{{{p}}({\bf{x}}_t}) {\textcolor{blue}{+\nabla_{{\bf{x}}_t} \log{p({\bf{y}}| {\bf{x}}_t})}})  \big)\Delta_t, 
    \label{reverse-process-post-form}
\end{align}

\noindent \textbf{Step 3}: Replace the prior score $\nabla_{{\bf{x}}_t} \log{{{p}}({\bf{x}}_t})$ back in terms of the generative model outputs in the obtained sampling equations in Step 2. Subsequently, taking into account the additional terms due to the addition of likelihood score, we can easily obtain the final posterior sampling equations. 
For example, for DDPM and flow-based models, the corresponding sampling equations (\ref{reverse-process-post-form-ddpm}-\ref{reverse-process-post-form}) finally become
\begin{align}
\text{DDPM:\;\;} \mathbf{x}_{t-1} &= \underbrace{\frac{1}{\sqrt{\alpha_t}} \left( \mathbf{x}_t - \frac{1-\alpha_t}{\sqrt{1-\bar{\alpha}_t}} s_{\theta}(\mathbf{x}_t, t) \right) + \beta_t \mathbf{z}_t}_{\text{original sampling equation}} + \textcolor{blue}{\underbrace{\frac{1-\alpha_t}{\sqrt{\alpha_t}} \nabla_{\mathbf{x}_t} \log p(\mathbf{y}|\mathbf{x}_t)}_{\textcolor{blue}{\text{additional part}}}}, \\
\text{Flow-based:\;\;}  \mathbf{x}_{t-1} &= \underbrace{\mathbf{x}_t - \mathbf{v}_{\theta}(\mathbf{x}_t, t) \Delta_t}_{\text{original sampling equation}} - \textcolor{blue}{\underbrace{\frac{b_t (\dot{a}_t b_t - a_t \dot{b}_t)}{a_t} \nabla_{\mathbf{x}_t} \log p(\mathbf{y}|\mathbf{x}_t) \Delta_t}_{\textcolor{blue}{\text{additional part}}}},
\end{align}
where the blue part is the addition terms required to incorporate into the original sampling equations to enable posterior sampling from  $p({\bf{x}}_0 |{\bf{y}} )$ given $\bf{y}$.  

Following the above procedures, we obtain the resultant algorithms for DDPM and flow-based models, as  shown in Algorithm \ref{alg: posterior-sampling-algorithm} and Algorithm \ref{alg: flow-based}, respectively. For brevity, we call both algorithms as Diffusion Model  based Posterior Sampling (dubbed DMPS) since flow-based models can be viewed as a generalization of diffusion models \cite{albergo2023stochastic}.  In the DDPM version, the reverse diffusion variance $\{\tilde{\sigma}_t\}_{t=1}^T$ is learned as the ADM in  \cite{dhariwal2021diffusion}. Both the two versions of DMPS algorithms can be  easily implemented on top of the existing  code just by adding two additional simple lines (lines 4-5 in Algorithm \ref{alg: posterior-sampling-algorithm}, lines 8-9 in Algorithm \ref{alg: flow-based}) of codes.

\begin{figure}[t!]
\centering
\begin{minipage}[t]{0.45\textwidth}
    \begin{algorithm}[H]
    \caption{DMPS (DDPM version)}
    \label{alg: posterior-sampling-algorithm}
    \DontPrintSemicolon
      \KwInput{$\bf{y,A}$, $\sigma_y^2$, $\{\tilde{\sigma}_t\}_{t=1}^T,\lambda$}
      \KwInitialize{${\bf{x}}_T\sim \mathcal{N}(\bf{0}, \bf{I})$, ${\bf{A} = U\Sigma V}^T$}
      \For{$t=T$ {\bfseries to} $1$}{
            Draw ${\bf{z}}_t \sim \mathcal{N}(\bf{0}, \bf{I})$ 
            
            ${\bf{x}}_{t-1} = \frac{1}{\sqrt{\alpha_t}} \big( {\bf{x}}_{t} - \frac{1-\alpha_t}{\sqrt{1-\bar{\alpha}}_t} {\rm{s}}_{\boldsymbol{\theta}}({\bf{x}}_{t},t)\big) + \tilde{\sigma}_t {\bf{z}}_t$  
      
            {{$\nabla_{{\bf{x}}_t} \log{\tilde{p}({\bf{y}}| {\bf{x}}_t}) \\=
            \frac{1}{\sqrt{{\bar{\alpha}}_t}} {\bf{V\Sigma}} {\Big(\sigma_y^2{\bf{I}}+ \frac{1-{\bar{\alpha}}_t}{{\bar{\alpha}}_t} {\bf{\Sigma}}^2 \Big)^{-1}} {\bf{U}}^T\big({\bf{y}} - \frac{1}{\sqrt{{\bar{\alpha}}_t}} {\bf{A}} {{{\bf{x}}}_t}\big)$} }     
            
            ${{\bf{x}}_{t-1} = {\bf{x}}_{t-1} + \lambda \frac{1-\alpha_t}{\sqrt{\alpha_t}} \nabla_{{\bf{x}}_t} \log{\tilde{p}({\bf{y}}| {\bf{x}}_t})}$
       }
    \KwOutput{${\bf{{x}}}_0$}
    \end{algorithm}
\end{minipage}
\hfill
\begin{minipage}[t]{0.45\textwidth}
    \begin{algorithm}[H]
   \caption{DMPS (flow-based version)}
    \label{alg: flow-based}
      \DontPrintSemicolon
      \KwInput{$\bf{y,A}$, $\sigma_y^2$, $\Delta_t=1/T$, $\lambda$}
      \KwInitialize{${\bf{x}}_T\sim \mathcal{N}(\bf{0}, \bf{I})$, ${\bf{A} = U\Sigma V}^T$}
      \For{$t=T$ {\bfseries to} $1$}{
            \vspace{\baselineskip}
            $\mathbf{x}_{t-1} = \mathbf{x}_{t} - {\rm{v}}_{\boldsymbol{\theta}}({\bf{x}}_{t},t)\Delta_t $  
      
            {{$\nabla_{{\bf{x}}_t} \log{\tilde{p}({\bf{y}}| {\bf{x}}_t}) \\=
            \frac{1}{a_t} {\bf{V\Sigma}} {\Big(\sigma_y^2{\bf{I}}+ \frac{b_t^2}{a^2_t} {\bf{\Sigma}}^2 \Big)^{-1}} {\bf{U}}^T\big({\bf{y}} - \frac{1}{\sqrt{{\bar{\alpha}}_t}} {\bf{A}} {{{\bf{x}}}_t}\big)$} }     
            
            ${{\bf{x}}_{t-1} = {\bf{x}}_{t-1} - \lambda \frac{b_t(\dot{a}_t b_t - a_t \dot{b}_t)}{a_t} \log{\tilde{p}({\bf{y}}| {\bf{x}}_t})\Delta_t }$
       }
    \KwOutput{${\bf{{x}}}_0$}
    \end{algorithm}
\end{minipage}
\end{figure}

\noindent {\textbf{Remark}}:  A scaling parameter $\lambda>0$ is introduced in both algorithms, similar to classifier guidance diffusion sampling \cite{dhariwal2021diffusion}. Empirically it is found that the performances are robust to different choices of $\lambda$  as shown in the Appendix \ref{appendix:effect-lambda}, and we fix $\lambda=1.75$ for DMPS (DDPM version) and  $\lambda=2.0$ for DMPS (flow-based version) in all the experiments.

\section{Experiments}
In this section, we conduct experiments on a variety of noisy linear inverse problems to demonstrate the efficacy of the proposed DMPS method, for both diffusion models and flow-based models. The code is available at \href{https://github.com/mengxiangming/dmps}{https://github.com/mengxiangming/dmps}.

\subsection{Experimental Setup}
\textbf{Tasks}:  The tasks we consider include image super-resolution (SR), denoising, deblurring, as well as image colorization. In particular: (a) for image super-resolution (SR), the bicubic downsampling is performed as \cite{chung2022diffusion}; (b) for deblurring,  uniform blur of size $9\times9$ \cite{kawar2022denoising} (for DDPM) and Gaussian blur (for flow-based) are used; (c) for colorization, the grayscale image is obtained by averaging the red, green, and blue channels of each pixel \cite{kawar2022denoising}.  For all tasks, additive Gaussian noise $\bf{n}$ with $\sigma=0.05$ is added except the denoising task where a larger noise $\bf{n}$ with $\sigma=0.5$ is added. 

\noindent \textbf{Dataset}: Both  FFHQ \cite{karras2019style} and CelebA-HQ \cite{karras2018progressive} are considered. More results on FFHQ-cat, LSUN-bedroom, and AFHQ-cat can be found in the Appendix \ref{append-more-results}. 

% For different tasks, the popular Flickr Faces High Quality (FFHQ) dataset \cite{karras2019style} is considered as in \cite{chung2022diffusion}. 

\noindent \textbf{Pre-trained Diffusion Models}: For a fair comparison, we use the same  pre-trained model for all the different methods evaluated. For diffusion models, the pre-trained ADM model \cite{choi2021ilvr} is used, available in \href{https://drive.google.com/drive/folders/1jElnRoFv7b31fG0v6pTSQkelbSX3xGZh}{DDPM-checkpoint}. For flow-based models, we use the pre-trained rectified flow model \cite{liu2022flow}, which is available in 
\href{https://drive.google.com/file/d/1ryhuJGz75S35GEdWDLiq4XFrsbwPdHnF/edit}{flow-checkpoint}, and the forward process (\ref{forward-process}) is specified as  $a_t = 1-t$, $b_t=t$.

\noindent \textbf{Comparison Methods}:  We compare DMPS with the following  methods: DPS \cite{chung2022diffusion}, PGDM \cite{song2022pseudoinverse}, and the OT-ODE method \cite{pokle2023training}. Actually, OT-ODE can be viewed as the flow-based version of PGDM.  For DPS, we also compare two versions: one is the original DDPM version, the other is the flow-based version obtained following the procedures described in Section \ref{sec:DMPS}. 

\noindent \textbf{Metrics}: Three widely used metrics are considered, including the standard distortion metric peak signal noise ratio (PSNR) (dB), as well as  two popular perceptual metrics: structural similarity index measure (SSIM) \cite{wang2004image} and Learned Perceptual Image Patch Similarity (LPIPS) \cite{zhang2018unreasonable}.     

\noindent \textbf{GPU}: All results are run on a single  NVIDIA  Tesla V100.

\begin{table*}[h!]
\scriptsize % reduce font size by about 10%
\centering
\setlength{\tabcolsep}{-1pt}
\begin{adjustbox}{width=\linewidth,center}
\begin{tabular*}{\linewidth}{
  @{\extracolsep{\fill}}
  ccccccccccccc
}
\toprule
& \mc{3}{c}{\textbf{super-resolution}} & \mc{3}{c}{\textbf{deblur}} & \mc{3}{c}{\textbf{colorization}}  & \mc{3}{c}{\textbf{denoising}} \\
\cmidrule{2-4} \cmidrule{5-7} \cmidrule{8-10} \cmidrule{11-13} 
\textbf{Method} & {PSNR $\uparrow$} & {SSIM $\uparrow$} & {LPIPS $\downarrow$} 
& {PSNR $\uparrow$} & {SSIM $\uparrow$} & {LPIPS $\downarrow$} 
& {PSNR $\uparrow$} & {SSIM $\uparrow$} & {LPIPS $\downarrow$} 
& {PSNR $\uparrow$} & {SSIM $\uparrow$} & {LPIPS $\downarrow$} \\

\toprule
DMPS {(DDPM, ours)} &  \textbf{27.63} &	\textbf{0.8450} &	\textbf{0.2071}	& 
\textbf{27.26} &	{0.7644} &	{0.2222} &
\textbf{21.09} &	\textbf{0.9592}	& \textbf{0.2738} &	
\textbf{27.81} &	{0.8777}	& {0.2435} 
 \\

\midrule
DPS (DDPM) &  26.78 &	0.8391 &	0.2329 & 26.50 &	\textbf{0.8151} &	0.2248 &	11.53 &	0.7923 &	0.5755 & 27.22 &	\textbf{0.8969} &	{0.2428} 	 \\

PGDM &  {27.60} &	0.8345 &	0.2077 & 26.65 &	0.7458 &	\textbf{0.2196}  &	{12.15} &	0.8920 &	0.3969 	& 27.60 &	{0.8682} &	\textbf{0.2425} \\

% DDRM & 26.80 &	44.86 & 0.2203 & 26.24 &	43.16 &	\textbf{0.2243} &	20.61 &	38.92 &	\textbf{0.2658} &	26.91 &	45.51 &	0.2282 \\

% MCG &   18.12 &	87.64 & 0.520 & 23.84 &	72.77 &	0.3746 &      11.5 &	71.36 &	0.5753  &	11.12 &	121.05 &	0.7841 \\
\bottomrule
\end{tabular*}
\end{adjustbox}
\caption{Quantitative comparison (PSNR (dB), SSIM, LPIPS) of different algorithms for different tasks on FFHQ $256\times 256$-1k validation dataset. The same pre-trained DDPM model is used. }
\label{table:ffhq}
\end{table*}

\begin{table*}[h!]
\scriptsize % reduce font size by about 10%
\centering
\setlength{\tabcolsep}{-1pt}
\begin{adjustbox}{width=\linewidth,center}
\begin{tabular*}{\linewidth}{
  @{\extracolsep{\fill}}
  ccccccccccccc
}
\toprule
& \mc{3}{c}{\textbf{super-resolution}} & \mc{3}{c}{\textbf{deblur}} & \mc{3}{c}{\textbf{colorization}}  & \mc{3}{c}{\textbf{denoising}} \\
\cmidrule{2-4} \cmidrule{5-7} \cmidrule{8-10} \cmidrule{11-13} 
\textbf{Method} & {PSNR $\uparrow$} & {SSIM $\uparrow$} & {LPIPS $\downarrow$} 
& {PSNR $\uparrow$} & {SSIM $\uparrow$} & {LPIPS $\downarrow$} 
& {PSNR $\uparrow$} & {SSIM $\uparrow$} & {LPIPS $\downarrow$} 
& {PSNR $\uparrow$} & {SSIM $\uparrow$} & {LPIPS $\downarrow$} \\

\toprule
DMPS (Flow-based, ours) & \textbf{28.29} & \textbf{0.8011} & 0.2329 & \textbf{26.21} & \textbf{0.7235} & {0.2637} & \textbf{23.31} & \textbf{0.8861} & \textbf{0.2901} & \textbf{29.04} & \textbf{0.8166} & \textbf{0.2821} \\
\midrule
DPS (Flow-based) & 28.05 & 0.7754 & \textbf{0.2266} & 22.64 & 0.5787 & 0.3403 & 20.92 & 0.8061 & 0.3335 & 27.93 & 0.7465 & 0.2882 \\

OT-ODE & 27.71 & 0.7657 & 0.2302 & 25.84 & 0.7084 & \textbf{0.2573} & 21.67 & 0.8696 & 0.3094 & 22.76 & 0.3820 & 0.4778 \\

\bottomrule
\end{tabular*}
\end{adjustbox}
\caption{Quantitative comparison (PSNR (dB), SSIM, LPIPS) of different algorithms for different tasks on the validation set of CelebA-HQ. The same pre-trained flow-based model is used.}
\label{table:celeba-hq}
\end{table*}

% \begin{table*}
% \scriptsize % reduce font size by about 10%
% \centering
% \setlength{\tabcolsep}{-1pt}
% \begin{adjustbox}{width=\linewidth,center}
% \begin{tabular*}{\linewidth}{
%   @{\extracolsep{\fill}}
%   ccccccccccccc
% }
% \toprule
% & \mc{3}{c}{\textbf{super resolution}} & \mc{3}{c}{\textbf{gaussian deblur}} & \mc{3}{c}{\textbf{colorization}}  & \mc{3}{c}{\textbf{denoising}} \\
% \cmidrule{2-4} \cmidrule{5-7} \cmidrule{8-10} \cmidrule{11-13} 
% \textbf{Method} & {PSNR $\uparrow$} & {SSIM $\uparrow$} & {LPIPS $\downarrow$} 
% & {PSNR $\uparrow$} & {SSIM $\uparrow$} & {LPIPS $\downarrow$} 
% & {PSNR $\uparrow$} & {SSIM $\uparrow$} & {LPIPS $\downarrow$} 
% & {PSNR $\uparrow$} & {SSIM $\uparrow$} & {LPIPS $\downarrow$} \\

% \toprule
% dps & 28.10 & 0.7132 & \textbf{0.2255} & 22.64 & 0.5152 & 0.3402 & 20.93 & 0.7416 & 0.3333 & 27.93 & 0.6905 & 0.2882 \\

% dmps & \textbf{28.29} & \textbf{0.7347} & 0.2328 & \textbf{26.21} & \textbf{0.6564} & \textbf{0.2637} & \textbf{23.31} & \textbf{0.8298} & 0.2898 & 29.04 & 0.7465 & \textbf{0.2819} \\

% OT-ODE & 27.72 & 0.7003 & 0.2298 & 25.84 & 0.6399 & 0.2571 & 21.69 & 0.8089 & 0.3082 & 22.77 & 0.3537 & 0.4781 \\

% \bottomrule
% \end{tabular*}
% \end{adjustbox}
% \caption{Quantitative comparison (PSNR, SSIM, LPIPS) of different algorithms for different tasks on the CelebA-HQ dataset.}
% \label{table:celeba-hq}
% \end{table*}

\begin{figure*}[!h]
\centering
\subfigure[Super-resolution (SR) ($\times 4$)]{
    \begin{minipage}[b]{0.48\textwidth}
    \includegraphics[width=\textwidth]{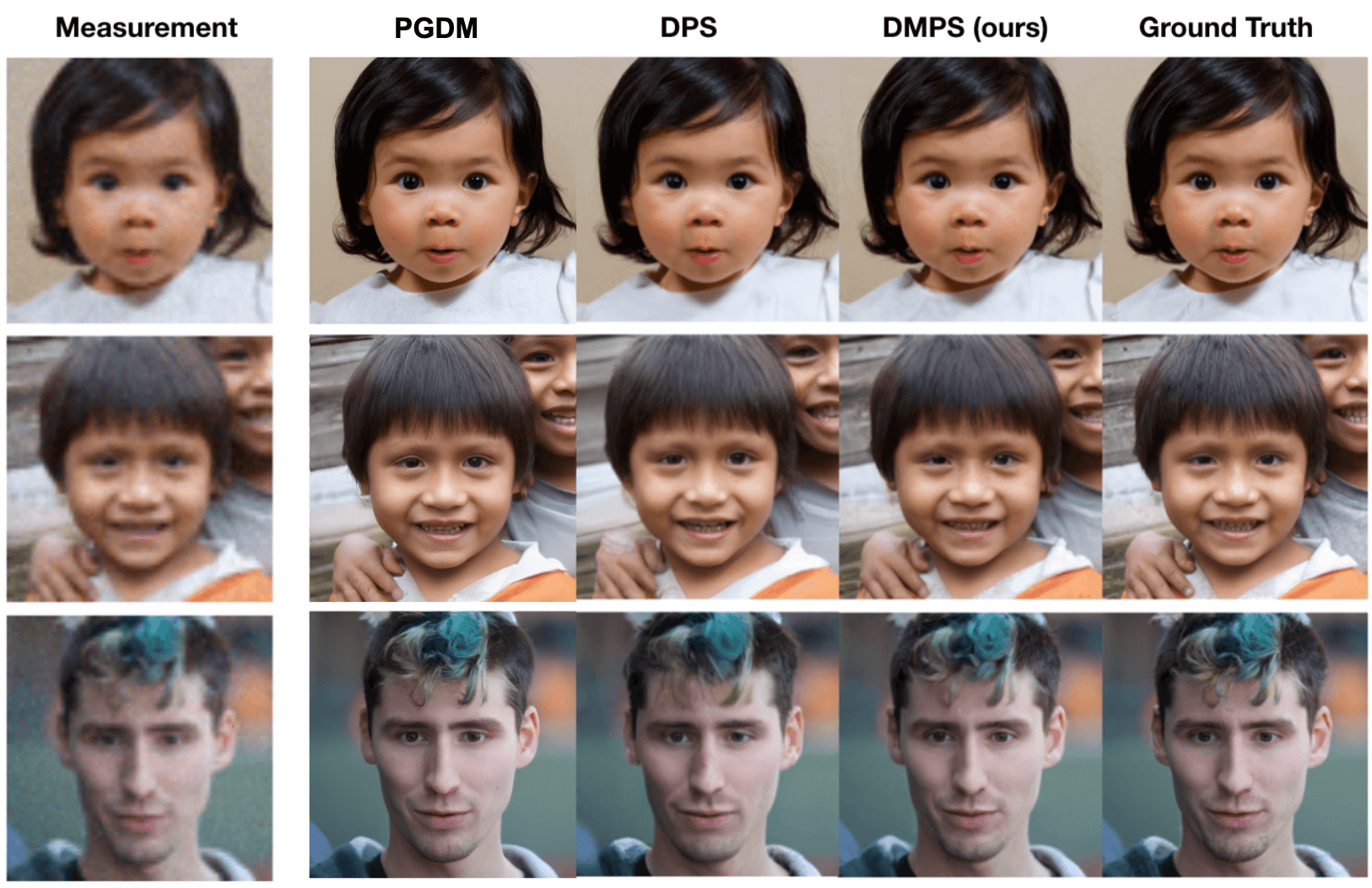}
    \end{minipage}
}
\subfigure[Denoising ($\sigma=0.5$)]{
  \begin{minipage}[b]{0.48\textwidth}
    \includegraphics[width=\textwidth]{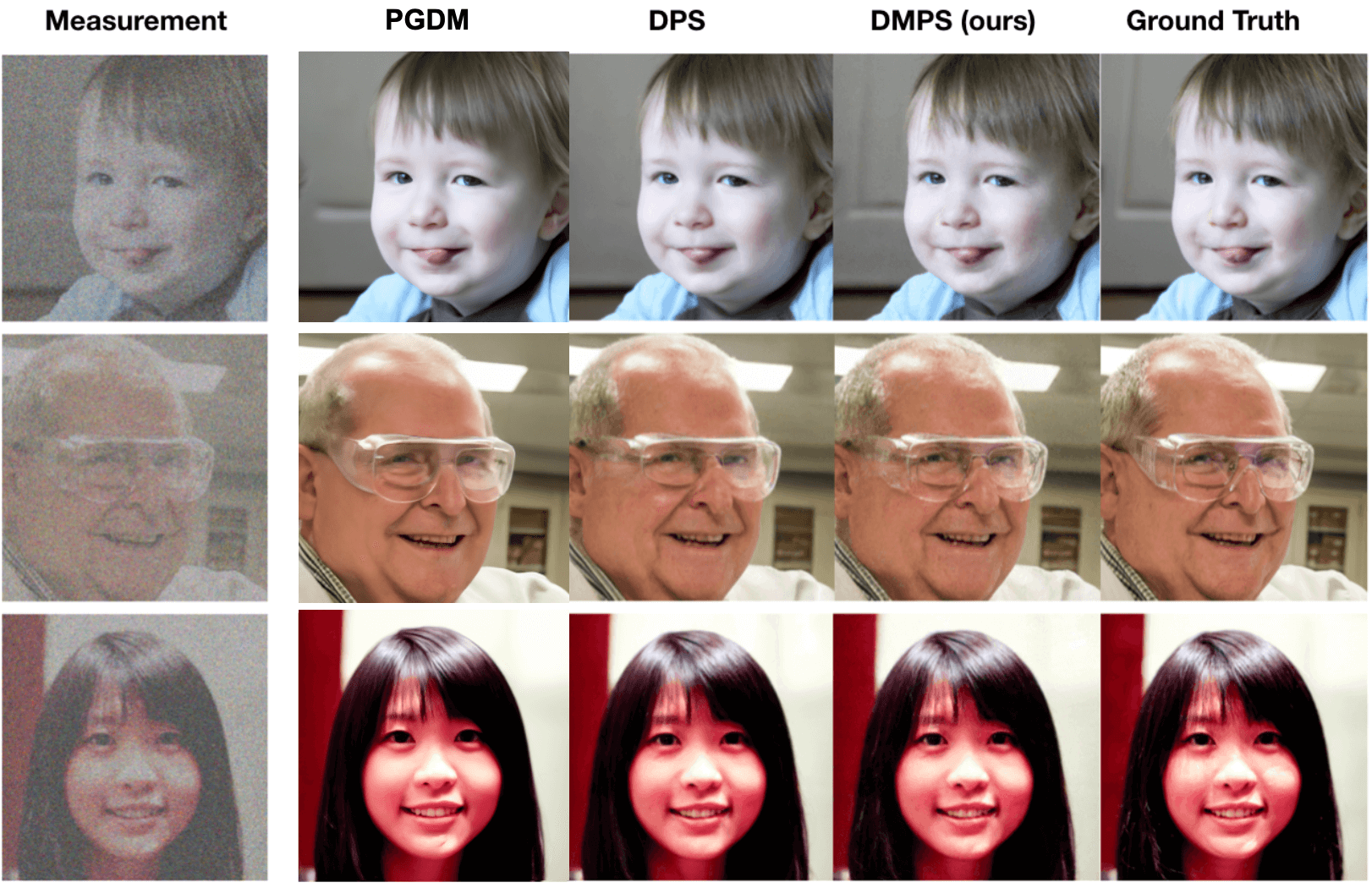}
  \end{minipage}
}
\subfigure[colorization]{
    \begin{minipage}[b]{0.48\textwidth}
    \includegraphics[width=\textwidth]{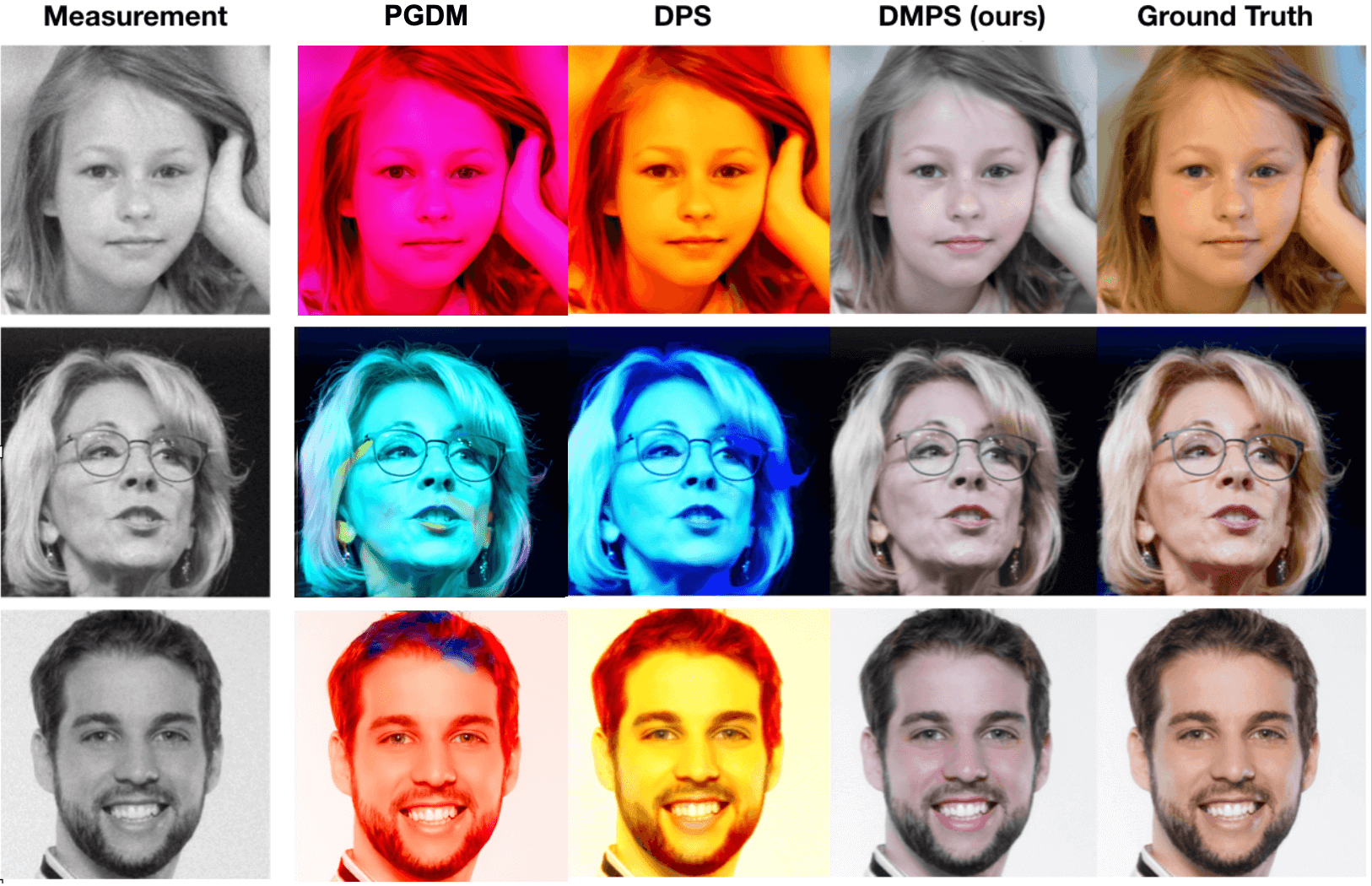}
    \end{minipage}
}
\subfigure[Deblurring (uniform)]{
  \begin{minipage}[b]{0.48\textwidth}
    \includegraphics[width=\textwidth]{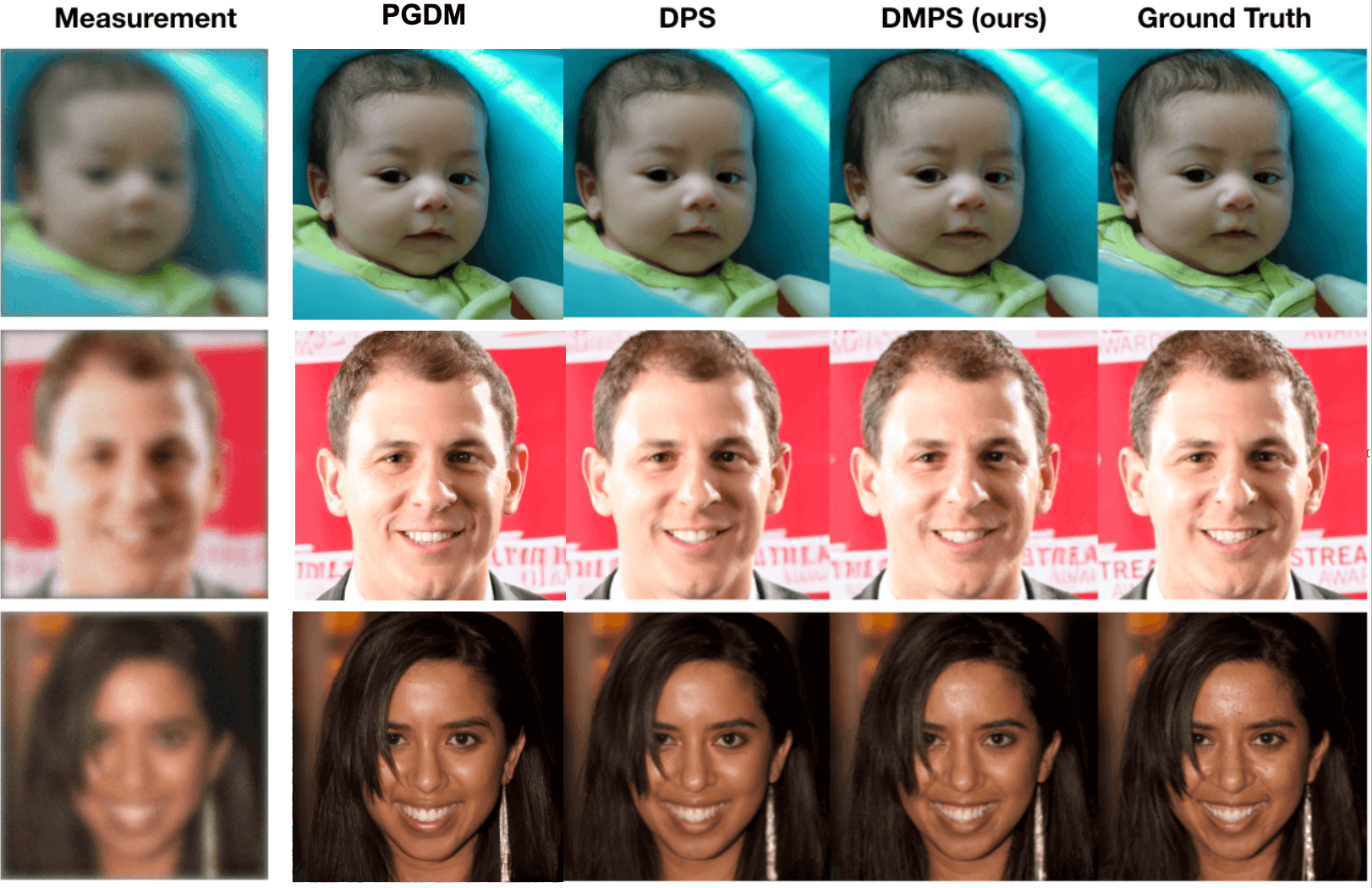}
  \end{minipage}
}
\caption{\small{Typical results on FFHQ $256\times 256$ 1k validation set for  different noisy linear inverse problems. All the algorithms are based on the same DDPM model. In all cases, the measurements are with Gaussian noise $\sigma=0.05$, except denoising where $\sigma=0.5$. }}
\label{fig-ffhq}
\end{figure*}

\noindent \textbf{Results}:  
First is a quantitative comparison in terms of different metrics.  Table \ref{table:ffhq} shows the quantitative reconstruction  performances of different algorithms on diffusion models on the FFHQ dataset, and  Table \ref{table:celeba-hq} shows the quantitative reconstruction performances of different algorithms on flow-based models with the CelebA-HQ dataset. As shown in Table \ref{table:ffhq} and Table \ref{table:celeba-hq}, despite its simplicity, the proposed DMPS achieves highly competitive or even better performances than the baselines. 

Second, we make a qualitative comparison between different algorithms for different tasks. Figure  \ref{fig-ffhq} shows the typical reconstructed images of different algorithms on diffusion models on the FFHQ dataset; Figure \ref{fig:flow-results} shows the quantitative reconstruction performances of different algorithms on flow-based models with the CelebA-HQ dataset. As shown in Figure  \ref{fig-ffhq} and Figure \ref{fig:flow-results}, in all tasks, DMPS produces high-quality realistic images which match details of the  ground-truth more closely. For example, for super-resolution, please have a look at the ear stud in the first row of Figure\ref{fig-ffhq} (a), the hand on the shoulder in the second row of Figure \ref{fig-ffhq} (a), and the background in the second  row of Figure  \ref{fig-ffhq} (a); for denoising, please see the background door in the first row of Figure \ref{fig-ffhq} (b), and the collar in the second row of Figure \ref{fig-ffhq} (b), and the last row of  Figure \ref{fig:flow-results}; for colorization, DPS tends to produce over-bright images in colorization while DMPS produces more natural colored images, as shown in Figure \ref{fig-ffhq} (c) and Figure \ref{fig:flow-results}, etc. 

\begin{figure}[t]
\centering
\begin{minipage}{0.45\textwidth}
  \centering
  \begin{tabular}{@{}lc@{}}
    \toprule
    Method &  Inference Time [s] \\
    \midrule
    DMPS {(DDPM, ours)} & \textbf{67.02} \\
    \midrule
    DPS (DDPM)  & 194.42 \\
    PGDM  & 182.35 \\
    \bottomrule
  \end{tabular}
\end{minipage}
\hfill
\begin{minipage}{0.45\textwidth}
  \centering
  \begin{tabular}{@{}lc@{}}
    \toprule
    Method &  Inference Time [s] \\
    \midrule
    DMPS {(flow-based, ours)} & \textbf{4.45} \\
    \midrule
    DPS (flow-based)  & 8.04 \\
    OT-DOE  & 6.44 \\
    \bottomrule
  \end{tabular}
\end{minipage}
\caption{Comparison of the inference time for different methods. Left: Results on DDPM models when NFE=1000, obtained on the SR task for  FFHQ $256\times 256$. Right: Results on flow-based models when NFE=50, obtained on the SR task for  CelebA-HQ $256\times 256$.  }
\label{fig:comparison-time}
\end{figure}

Finally, we evaluate the inference time of different algorithms, which is one of the key motivation of this paper. Here we would like to emphasize again  that the main goal of this paper is not to compete with state-of-the-art performance but rather to provide a fast method. For fair of comparison, for both diffusion and flow-based models, different algorithms uses the same pre-trained model.  Figure \ref{fig:comparison-time} show the the average running time for different algorithms: Left table shows the results under diffusion models when the number of function evaluation (NFE) is NFE = 1000; Right Table shows the results of different algorithms under flow-based models when NFE = 50. It can be seen that, in both versions, the inference time of the proposed DMPS method is significantly less than other methods, which is much appealing in practical applications.

\section{Discussion and Conclusion}
\label{sec:Conclusion}
In this paper, we propose fast and effective closed-form approximation of the intractable noise-perturbed likelihood score, leading to the Diffusion Model based Posterior Sampling (dubbed DMPS). 
For both diffusion and flow-based models, we evaluate the effectiveness of DMPS on multiple linear inverse problems including image super-resolution, denoising, deblurring, colorization. Despite its simplicity, DMPS achieves highly competitive or even better reconstruction performances, while its inference time of DMPS is significantly faster. 

\noindent\textbf{Limitations $\&$ Future Work}: While DMPS apparently reduces the inference time and achieves competitive reconstruction performances, it still suffers several limitations. First, although memory efficient SVD exists for most practical matrices  $\bf{A}$ of practical interests \cite{kawar2022denoising}, the SVD operation in DMPS still has some implementation difficulty for more general matrices $\bf{A}$. Second, it can not be directly applied to the popular latent diffusion models such as stable diffusion \cite{rombach2022high}, which is widely used due to its efficiency. Addressing these limitations are left as  future work.

\section*{Acknowledgements}
X. Meng would like to sincerely thank Yichi Zhang and Jim Yici Yan from UIUC for helpful discussions. 
This work was supported by NSFC No. 62306277,  and the Fundamental Research Funds for the Zhejiang Provincial Universities Grant No. K20240090, The Japan Science and Technology Agency (JST) Grant No. JPMJCR1912, and The Japan Society for the Promotion of Science (JSPS) Grant No. JP22H05117.

\bibliography{acml24}

\appendix

\section{Verification of Assumption \ref{uninformative-assumption}}
\label{appendix:uninformative-prior}
Here we provide a theoretical support of the uninformative prior assumption \ref{uninformative-assumption}, or, equivalently,  the following Gaussian approximation of the posterior $p(\mathbf{x} _ 0 | \mathbf{x} _ t)$: 
\begin{align}
p(\mathbf{x} _ 0 | \mathbf{x} _ t) \approx   \mathcal{N}(\frac{\mathbf{x} _ t}{a_t} , b_t^2 \bf{I}), \label{eq-original}
\end{align}
Throughout the following derivations, we will drop any additive constants in the log (which translate to normalizing factors), and drop all terms of order $O(t)$.

Let us start with the original Bayes' formula (using the log form):

\begin{align}
\log p(\mathbf{x}_0|\mathbf{x}_t) = \log p(\mathbf{x}_t|\mathbf{x}_0) + \log p_0(\mathbf{x}_0) - \log p_t(\mathbf{x}_t),
\end{align}
where $p _ 0 (\mathbf{x}_0) $ and $ p_t(\mathbf{x} _ t) $ denote the marginal distribution of $ \mathbf{x} _ 0 $ and $ \mathbf{x} _ t $, respectively. 

Since $ p_{t-\Delta t}(\cdot) = p_t(\cdot) + \Delta t \frac{\partial}{\partial t} p_t(\cdot) + \mathcal{O}(\Delta t) $  for $ |t| \ll 1 $, there is
\begin{align}
\log p_0(\mathbf{x}_0|\mathbf{x}_t) = \log p(\mathbf{x}_t|\mathbf{x}_0) + \log p_t(\mathbf{x}_0) + \mathcal{O}(t) - \log p_t(\mathbf{x}_t). \label{11}
\end{align}

For (\ref{11}), we perform a first order Taylor expansion of $\log p_t(\mathbf{x}_0) $ around $\mathbf{x}_t $, which yields
\begin{align}
\log p_0(\mathbf{x}_0|\mathbf{x}_t) &= \log p(\mathbf{x}_t|\mathbf{x}_0) + \log p_t(\mathbf{x}_t) + \langle \nabla _ {\mathbf{x}_t} \log p_t(\mathbf{x}_t), \mathbf{x}_0 - \mathbf{x}_t \rangle + \mathcal{O}(t) - \log p_t(\mathbf{x}_t) \nonumber \\
&= \log p(\mathbf{x}_t|\mathbf{x}_0) + \langle \nabla _ {\mathbf{x}_t} \log p_t(\mathbf{x}_t), \mathbf{x}_0 - \mathbf{x}_t \rangle + \mathcal{O}(t). 
\end{align}
Substituting  $ p(\mathbf{x}_t | \mathbf{x}_0) = \mathcal{N}(a_t \mathbf{x}_0, b_t^2) $ and completing the squares, we obtain:
\begin{align}
\log p(\mathbf{x}_0 | \mathbf{x} _ t) &= -\frac{\lVert \mathbf{x}_t - a_t \mathbf{x}_0 \rVert^2}{2b_t^2} + \langle \nabla _ {\mathbf{x}_t} \log p_t (\mathbf{x}_t), \mathbf{x}_0 - \mathbf{x}_t \rangle + \mathcal{O}(t)
&=  -\frac{1}{2b_t^2} \lVert \mathbf{x}_0 - \boldsymbol{\mu} \rVert^2 + C,
\end{align}
where $C$ is a constant value and the mean value $\boldsymbol{\mu} $ is:
\begin{align}
\boldsymbol{\mu} = \frac{\mathbf{x} _ t}{a_t} + \frac{b_t^2}{a_t^2} \nabla_{\mathbf{x}_t} \log p_t (\mathbf{x}_t)
\end{align}
Therefore, we obtain that the posterior distribution $p(\mathbf{x} _ 0 | \mathbf{x} _ t)$ can be approximated as a Gaussian 
\begin{align}
p(\mathbf{x} _ 0 | \mathbf{x} _ t) \approx   \mathcal{N}(\frac{\mathbf{x} _ t}{a_t} + \frac{b_t^2}{a_t^2} \nabla_{\mathbf{x}_t} \log p_t (\mathbf{x}_t), b_t^2 \bf{I}) \label{eq-final}
\end{align}

Comparing eqs.  (\ref{eq-original}) and  (\ref{eq-final}), we can see that in our result (1), we further ignore the term $\frac{b_t^2}{a_t^2} \nabla_{\mathbf{x}_t} \log p_t (\mathbf{x}_t)$ in the mean value. This is valid for sufficiently small $t$ since the variance $b^2_t$ is sufficiently small following the special design principle in forward diffusion process. For example, for DDPM and flow-based model considered in our manuscript, $b^2_t = 1- \bar{\alpha} _ t $, $b^2_t = t^2$, respectively.

Reflecting on this derivation, the main idea is that for a sufficiently small  $ t$, the Bayes' rule expansion of $p(\mathbf{x}_{0} \mid \mathbf{x}_t)$ (recall that this is what we need to compute the likelihood score) is primarily influenced by the term  $p(\mathbf{x} _ t \mid \mathbf{x} _ {0})$ from the forward process, regardless of the prior of $p(\mathbf{x} _ {0})$. As a result, the uninformative prior assumption  is reasonable for sufficiently small $t$.  In fact, this  insight is exactly why in the diffusion models the reverse process and the forward process  share the same functional form for sufficiently small time interval. It is worth pointing out that, the validity of the above results does not depend on the underlying distribution $p_0(\mathbf{x} _ 0)$, whether it being a simple Gaussian or a complex distribution as that of a face image.

\begin{figure}
\centering  
\includegraphics[width=0.9\textwidth]{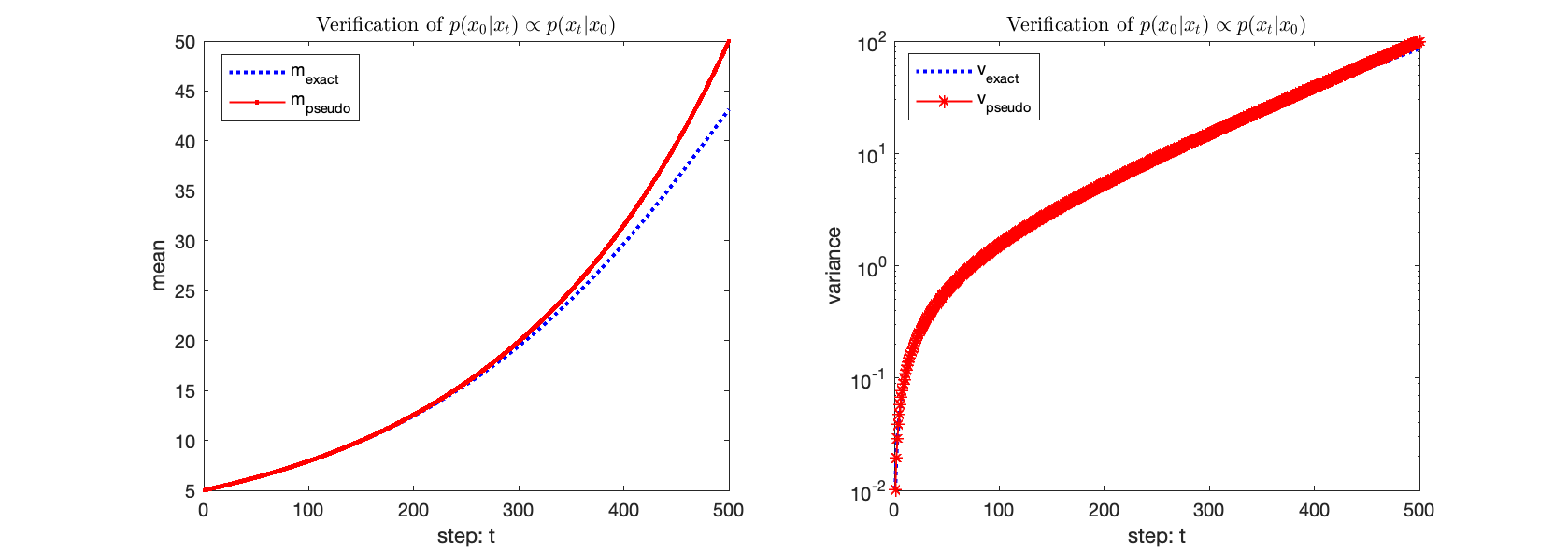}
\caption{Comparison of the exact mean and variance of  $p(  {{x}}_0 | {{x}_t})$  with the pseudo mean and variance under the uninformative assumption, i.e.,  $p(  {{x}_0} | {{x}_t}) \propto p( {{x}}_t |   {{x}_0} ) $. }
\label{fig:uninformative-toy}
\end{figure}

\noindent \textbf{A toy example:} We further consider a toy  example to illustrate this where  the exact form of  $p({\bf{x}}_0 | {\bf{x}}_t )$ in (\ref{eq:pdf_condition_x}) can be computed exactly. Assume that  $\bf{x}$ reduces to a scalar random variable $x$ and the associated prior $p(x)$ follows a Gaussian distribution, i.e., $p(x) = \mathcal{N}(x;0,\sigma^2_0)$, where $\sigma^2$ is the prior variance. The likelihood $p({\bf{x}}_t ) | {\bf{x}}_0)$ (\ref{forward-process}) in this case is simply $p({{x}}_t | {{x}}_0) =  \mathcal{N}({x_t};\sqrt{{\bar{\alpha}}_t} x_0,(1-\bar{\alpha}_t))$.

Then, from (\ref{eq:pdf_condition_x}), after some algebra, it can be computed that the posterior distribution $p({{x}}_0 | {{x}_t})$ is 
\begin{align}
    p(  {{x}_0} | {{x}_t}) = \mathcal{N}({x}_0;m_\textrm{exact},v_\textrm{exact})
\end{align}
where 
\begin{align}
    m_\textrm{exact} = \frac{\sqrt{{\bar{\alpha}}_t}\sigma_0^2}{(1-{\bar{\alpha}}_t) + {\bar{\alpha}}_t \sigma^2_0} {x}_t, \;   v_\textrm{exact} = \frac{(1-{\bar{\alpha}}_t)\sigma^2_0}{(1-{\bar{\alpha}}_t) + {\bar{\alpha}}_t \sigma^2_0}. 
    \label{exact-mean-variance}
\end{align}
Under the Assumption \ref{uninformative-assumption}, i.e., $p(  {{x}_0} | {{x}_t}) \propto p( {{x}_t} |  {{x}_0} ) $, we obtain an approximation of $p(  {{x}} | {{x}_t}) $ as follows
\begin{align}
    p(  {{x}_0} | {{x}_t}) \simeq  \tilde{p}(  {{x}_0} | {{x}_t}) =  \mathcal{N}({x}_0;m_\textrm{pseudo},v_\textrm{pseudo}),  
\end{align}
where 
\begin{align}
    m_\textrm{pseudo} = \frac{1}{\sqrt{{\bar{\alpha}}_t}} {x}_t, \;   v_\textrm{pseudo} = \frac{1-{{\bar{\alpha}}_t}}{\bar{\alpha}}. \label{pseudo-mean-variance}
\end{align}
By comparing the exact result (\ref{exact-mean-variance}) and approximation result (\ref{pseudo-mean-variance}), it can be easily seen that for a fixed $\sigma^2_0>0$, as $\bar{\alpha}_t\to 1$, we have $m_\textrm{pseudo} \to m_\textrm{post}$ and  $v_\textrm{pseudo} \to v_\textrm{post}$, which is exactly the case for DDPM as $t\to 1$. To see this, we anneal $\bar{\alpha}_t$ as $\bar{\alpha}_t = {\bar{\alpha}}_{\rm{max}} (\frac{{\bar{\alpha}}_{\rm{min}}}{{\bar{\alpha}}_{\rm{max}}})^{\frac{t-1}{T-1}}$ geometrically and compare $ m_\textrm{pseudo},  v_\textrm{pseudo}$ with $m_\textrm{exact}, v_\textrm{exact}$ as $t$ increase from $1$ to $T$. Assume that ${\bar{\alpha}}_{\rm{min}} = 0.01$ and ${\bar{\alpha}}_{\rm{min}} = 0.99$, and $\sigma_0 = 25,  x_t = 5,  T = 500$, we obtain the results in Fig. \ref{fig:uninformative-toy}.  It can be seen in Fig. \ref{fig:uninformative-toy} that the approximated values $ m_\textrm{pseudo},  v_\textrm{pseudo}$, especially the variance $v_\textrm{pseudo}$,  approach to the exact values $m_\textrm{exact}, v_\textrm{exact}$ very quickly, verifying the effectiveness of the Assumption \ref{uninformative-assumption} under this toy example.

% \section{Proof of Theorem \ref{noise-likelohood-score-linear-VPSDE}}
% \label{appendix:proof-of-theorem1}

% \begin{center}
%     \Large \textbf{Supplementary Material for ``Diffusion Model Based Posterior Sampling for Noisy Linear Inverse Problems''}
% \end{center}

% In this supplementary material, we provide results on the effect of scaling parameter $\lambda$, as well as results on more datasets. 

\section{Effect of Scaling Parameter $\lambda$}
\label{appendix:effect-lambda}
As shown in both Algorithm \ref{alg: posterior-sampling-algorithm} and Algorithm \ref{alg: flow-based}, a hyper-parameter $\lambda$ is introduced as a scaling value for the likelihood score.  Empirically it is found that DMPS is robust to different choices of $\lambda$ around 1 though most of the time $\lambda>1$ yields slightly better results. As one specific example, we show the results of DMPS for super-resolution for different values of $\lambda$, as shown in Figure \ref{fig:effect-of-lambda} (DDPM version) and Figure \ref{fig:effect-of-lambda-flow} (flow-based version). It can be seen that DMPS is robust to different choices of $\lambda$, i.e., it works  well in a wide range of values.

\begin{figure*}[!h]
\centering  
\includegraphics[width=0.9\textwidth]{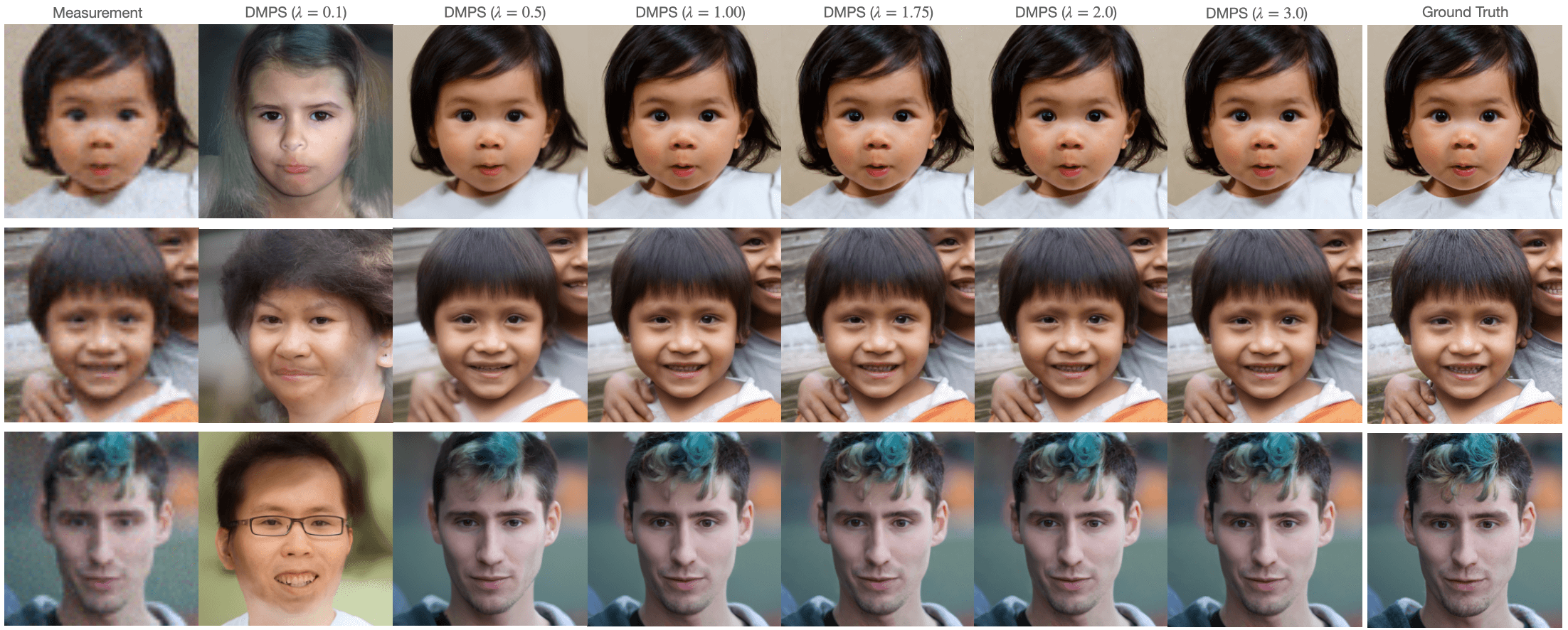}
\caption{Results of DMPS (DDPM version) with different $\lambda$ for the task of noisy super-resolution ($\times 4$) with $\sigma = 0.05$.}
\label{fig:effect-of-lambda}
\end{figure*}

\begin{figure*}[!h]
\centering  
\includegraphics[width=0.9\textwidth]{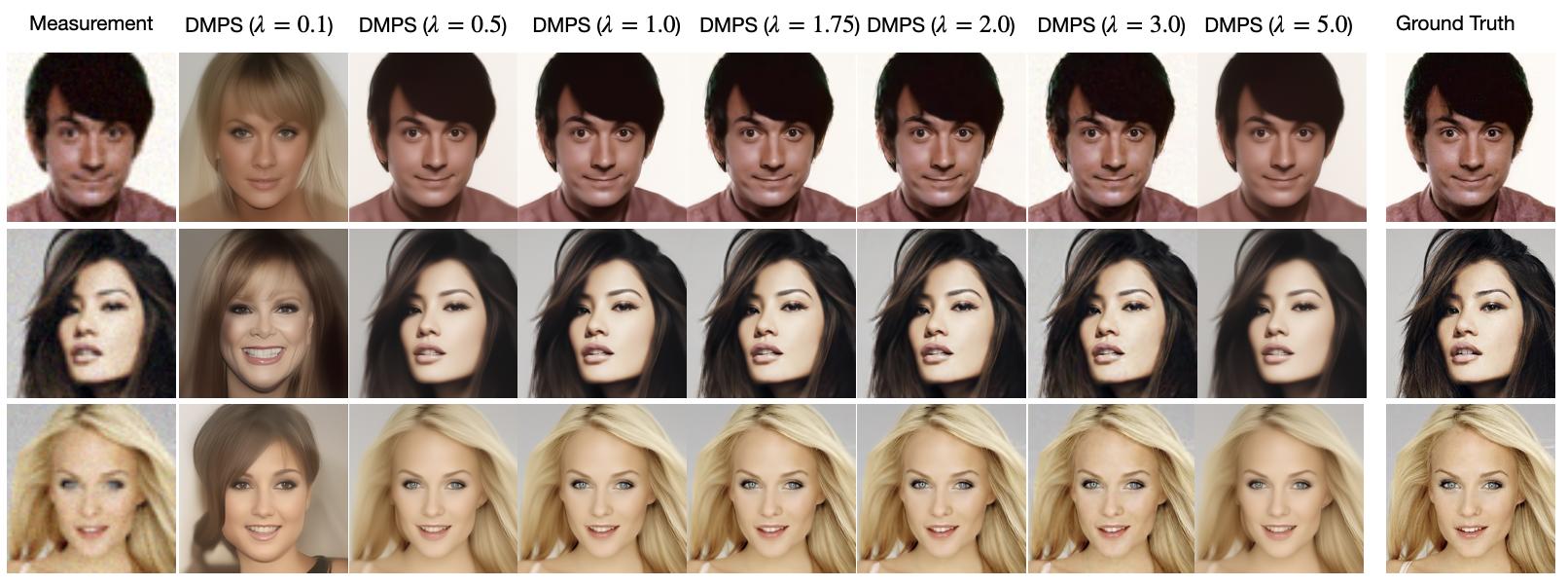}
\caption{Results of DMPS (flow-based version) with different $\lambda$ for the task of noisy super-resolution ($\times 4$) with $\sigma = 0.05$.}
\label{fig:effect-of-lambda-flow}
\end{figure*}

\section{Results on More Datasets}
\label{append-more-results}

We provide more experimental results on AFHQ-cat and LSUN-bedroom for flow-based models are shown as follows:

\begin{table*}[h!]
\scriptsize % reduce font size by about 10%
\centering
\setlength{\tabcolsep}{-1pt}
\begin{adjustbox}{width=\linewidth,center}
\begin{tabular*}{\linewidth}{
  @{\extracolsep{\fill}}
  ccccccccccccc
}
\toprule
& \mc{3}{c}{\textbf{super-resolution}} & \mc{3}{c}{\textbf{deblur}} & \mc{3}{c}{\textbf{colorization}}  & \mc{3}{c}{\textbf{denoising}} \\
\cmidrule{2-4} \cmidrule{5-7} \cmidrule{8-10} \cmidrule{11-13} 
\textbf{Method} & {PSNR $\uparrow$} & {SSIM $\uparrow$} & {LPIPS $\downarrow$} 
& {PSNR $\uparrow$} & {SSIM $\uparrow$} & {LPIPS $\downarrow$} 
& {PSNR $\uparrow$} & {SSIM $\uparrow$} & {LPIPS $\downarrow$} 
& {PSNR $\uparrow$} & {SSIM $\uparrow$} & {LPIPS $\downarrow$} \\

\toprule
DMPS (DDPM, ours) & \textbf{26.79} & \textbf{0.7653} & \textbf{0.2632} & \textbf{27.22} & \textbf{0.7571} & \textbf{0.2909} & \textbf{25.07} & \textbf{0.9190} & \textbf{0.3124} & 28.59 & \textbf{0.7994} & \textbf{0.2882} \\
\midrule
DPS (DDPM) & 23.08 & 0.6127 & 0.3860 & 24.64 & 0.6625 & 0.3033 & 15.92 & 0.5976 & 0.6381 & \textbf{28.86} & 0.7828 & 0.2941 \\
PGDM & 25.44 & 0.7185 & 0.2837 & 26.69 & 0.7316 & 0.2896 & 16.74 & 0.6348 & 0.5335 & 27.06 & 0.7453 & 0.3236 \\
\bottomrule
\end{tabular*}
\end{adjustbox}
\caption{Results on FFHQ-Cat validation dataset using the same pre-trained DDPM model.}
\label{table:ffhq-cat}
\end{table*}

\begin{table*}[h!]
\scriptsize % reduce font size by about 10%
\centering
\setlength{\tabcolsep}{-1pt}
\begin{adjustbox}{width=\linewidth,center}
\begin{tabular*}{\linewidth}{
  @{\extracolsep{\fill}}
  ccccccccccccc
}
\toprule
& \mc{3}{c}{\textbf{super-resolution}} & \mc{3}{c}{\textbf{deblur}} & \mc{3}{c}{\textbf{colorization}}  & \mc{3}{c}{\textbf{denoising}} \\
\cmidrule{2-4} \cmidrule{5-7} \cmidrule{8-10} \cmidrule{11-13} 
\textbf{Method} & {PSNR $\uparrow$} & {SSIM $\uparrow$} & {LPIPS $\downarrow$} 
& {PSNR $\uparrow$} & {SSIM $\uparrow$} & {LPIPS $\downarrow$} 
& {PSNR $\uparrow$} & {SSIM $\uparrow$} & {LPIPS $\downarrow$} 
& {PSNR $\uparrow$} & {SSIM $\uparrow$} & {LPIPS $\downarrow$} \\

\toprule
DMPS (DDPM, ours) & \textbf{25.63} & \textbf{0.7362} & \textbf{0.2281} & \textbf{28.21} & \textbf{0.8162} & \textbf{0.2113} & \textbf{23.19} & \textbf{0.9344} & \textbf{0.2117} & 29.81 & 0.8599 & 0.1884 \\
\midrule
DPS (DDPM) & 22.83 & 0.6190 & 0.3275 & 24.97 & 0.6988 & 0.2593 & 11.38 & 0.5375 & 0.6606 & \textbf{30.75} & \textbf{0.8674} & \textbf{0.1841} \\
PGDM & 24.60 & 0.6854 & 0.2590 & 26.90 & 0.7721 & 0.2482 & 17.69 & 0.7335 & 0.3350 & 27.90 & 0.8153 & 0.2304 \\
\bottomrule
\end{tabular*}
\end{adjustbox}
\caption{Results on LSUN-Bedroom validation dataset using the same pre-trained DDPM model.}
\label{table:lsun-ddpm}
\end{table*}

\begin{table*}[h!]
\scriptsize % reduce font size by about 10%
\centering
\setlength{\tabcolsep}{-1pt}
\begin{adjustbox}{width=\linewidth,center}
\begin{tabular*}{\linewidth}{
  @{\extracolsep{\fill}}
  ccccccccccccc
}
\toprule
& \mc{3}{c}{\textbf{super-resolution}} & \mc{3}{c}{\textbf{deblur}} & \mc{3}{c}{\textbf{colorization}}  & \mc{3}{c}{\textbf{denoising}} \\
\cmidrule{2-4} \cmidrule{5-7} \cmidrule{8-10} \cmidrule{11-13} 
\textbf{Method} & {PSNR $\uparrow$} & {SSIM $\uparrow$} & {LPIPS $\downarrow$} 
& {PSNR $\uparrow$} & {SSIM $\uparrow$} & {LPIPS $\downarrow$} 
& {PSNR $\uparrow$} & {SSIM $\uparrow$} & {LPIPS $\downarrow$} 
& {PSNR $\uparrow$} & {SSIM $\uparrow$} & {LPIPS $\downarrow$} \\

\toprule
DMPS (Flow-based, ours) & \textbf{29.06} & \textbf{0.7905} & \textbf{0.2627} & \textbf{26.74} & \textbf{0.6942} & \textbf{0.3192} & 24.65 & \textbf{0.9140} & \textbf{0.2531} & \textbf{26.53} & \textbf{0.7870} & \textbf{0.3353} \\
\midrule
DPS (Flow-based) & 27.61 & 0.7089 & 0.3190 & 23.26 & 0.5534 & 0.4122 & 21.64 & 0.8259 & 0.3833 & 26.10 & 0.6418 & 0.4049 \\
OT-ODE & 27.61 & 0.7081 & 0.3205 & 26.32 & 0.6592 & 0.3333 & \textbf{25.21} & 0.8692 & 0.3180 & 23.12 & 0.3647 & 0.5289 \\
\bottomrule
\end{tabular*}
\end{adjustbox}
\caption{Results on AFHQ-Cat validation dataset using the same pre-trained flow-based model.}
\label{table:afhq-cat}
\end{table*}

\begin{table*}[h!]
\scriptsize % reduce font size by about 10%
\centering
\setlength{\tabcolsep}{-1pt}
\begin{adjustbox}{width=\linewidth,center}
\begin{tabular*}{\linewidth}{
  @{\extracolsep{\fill}}
  ccccccccccccc
}
\toprule
& \mc{3}{c}{\textbf{super-resolution}} & \mc{3}{c}{\textbf{deblur}} & \mc{3}{c}{\textbf{colorization}}  & \mc{3}{c}{\textbf{denoising}} \\
\cmidrule{2-4} \cmidrule{5-7} \cmidrule{8-10} \cmidrule{11-13} 
\textbf{Method} & {PSNR $\uparrow$} & {SSIM $\uparrow$} & {LPIPS $\downarrow$} 
& {PSNR $\uparrow$} & {SSIM $\uparrow$} & {LPIPS $\downarrow$} 
& {PSNR $\uparrow$} & {SSIM $\uparrow$} & {LPIPS $\downarrow$} 
& {PSNR $\uparrow$} & {SSIM $\uparrow$} & {LPIPS $\downarrow$} \\

\toprule
DMPS (Flow-based, ours) & \textbf{24.36} & \textbf{0.6795} & 0.3837 & \textbf{23.19} & \textbf{0.5869} & \textbf{0.4384} & \textbf{23.37} & \textbf{0.8756} & \textbf{0.2838} & 22.68 & \textbf{0.6477} & \textbf{0.4458} \\
\midrule
DPS (Flow-based) & 24.39 & 0.6430 & \textbf{0.3781} & 20.13 & 0.4318 & 0.4931 & 11.03 & 0.5283 & 0.7843 & 23.18 & 0.5457 & 0.4598 \\
OT-ODE & 23.88 & 0.6193 & 0.4001 & 22.69 & 0.5590 & 0.4264 & \textbf{23.62} & 0.7592 & 0.3923 & 18.17 & 0.2039 & 0.6405 \\
\bottomrule
\end{tabular*}
\end{adjustbox}
\caption{Results on LSUN-Bedroom validation dataset using the same pre-trained flow-based model.}
\label{table:lsun-bedroom}
\end{table*}

\end{document}